\documentclass[twoside]{article}
\usepackage{floatrow}
\usepackage{ifthen}
\usepackage{xspace}
\usepackage[T1]{fontenc}
\usepackage[OT1]{eulervm}
\usepackage[dvipsnames]{xcolor}
\usepackage[colorlinks=true,
            linkcolor=Blue,
            urlcolor=Blue,
            citecolor=Blue]{hyperref}
\usepackage{acro}
\acsetup{
    list-style     = tabular,
    single         = true ,
    sort           = true ,
    cite-cmd       = \citep ,
    group-citation = true ,
    group-cite-cmd = \citealp
}
\usepackage{ragged2e}
\usepackage{microtype}
\usepackage{dirtytalk}
\usepackage{pdflscape}
\usepackage{afterpage}
\usepackage{fancyvrb}
\usepackage{placeins}
\usepackage{csquotes}
\usepackage{filecontents}
\usepackage{url}            % simple URL typesetting
\usepackage{booktabs}       % professional-quality tables
\usepackage{nicefrac}       % compact symbols for 1/2, etc.
\usepackage{microtype}
\usepackage{graphicx}
\usepackage{nicefrac}
\usepackage{subcaption}
\usepackage[backend=bibtex, style=authoryear-comp, maxcitenames=1,
            natbib=true, bibencoding=utf8, hyperref=true,
            abbreviate=true, firstinits=true,
            backref=false, citetracker=true]{biblatex}
\VerbatimFootnotes
\usepackage[export]{adjustbox}
\usepackage{pgf}
\usepackage{pgfplots}
\usepackage{tikz}
\usetikzlibrary{pgfplots.groupplots}
\usetikzlibrary{external}
\usetikzlibrary{quotes,angles}
\usetikzlibrary{cd}
\usetikzlibrary{matrix}
\usetikzlibrary{calc}
\makeatletter
\tikzset{%
    column sep/.code=\def\pgfmatrixcolumnsep{\pgf@matrix@xscale*(#1)},
    row sep/.code   =\def\pgfmatrixrowsep{\pgf@matrix@yscale*(#1)},
    matrix xscale/.code=%
    \pgfmathsetmacro\pgf@matrix@xscale{\pgf@matrix@xscale*(#1)},
    matrix yscale/.code=%
    \pgfmathsetmacro\pgf@matrix@yscale{\pgf@matrix@yscale*(#1)},
    matrix scale/.style={/tikz/matrix xscale={#1},/tikz/matrix yscale={#1}}}
\def\pgf@matrix@xscale{1}
\def\pgf@matrix@yscale{1}
\makeatother
\pgfplotsset{compat=newest}
\usepackage{amsmath, amsthm, amssymb, amsfonts, bbm}
\usepackage[american]{babel}
\usepackage{braket}
\usepackage{colonequals}
\usepackage{thmtools}
\usepackage{nameref}
\usepackage[capitalize]{cleveref}
\usepackage{mathtools}
\usepackage[algo2e, ruled, linesnumbered]{algorithm2e}
\usepackage{tablefootnote}
\usepackage{breqn}
\usepackage{commands}
\usepackage{paralist, enumitem}
\usepackage{multicol}
\usepackage{multirow}
\usepackage{booktabs}
\usepackage{tabularx}
\usepackage{ltablex}
\usepackage{chngpage}
\usepackage{threeparttablex}
\usepackage[mathscr]{eucal}
\usepackage{floatflt}
\usepackage{enumitem}
\usepackage{times}
\creflabelformat{equation}{#2(\textup{#1})#3}
\usepackage[accepted]{aistats2019_without_decoration}
\begin{document}
\runningauthor{
Romain Brault, Alex Lambert, Zolt\'an Szab\'o, Maxime Sangnier, Florence d'Alch{\'e}-Buc
}
\twocolumn[
  \aistatstitle{Infinite Task Learning in RKHSs}
  \aistatsauthor{ Romain Brault $^\dagger$
                \And Alex Lambert $^\dagger$
                \And Zolt\'an Szab\'o
                \And Maxime Sangnier
                \And Florence d'Alch{\'e}-Buc}
  \aistatsaddress{L2S, \\
                  Centrale-Sup\'elec, \\
                  Universit\'e  Paris-Saclay, \\
                  Gif sur Yvette, France.
                  \And
                  LTCI, \\
                  T\'el\'ecom ParisTech, \\
                  Universit\'e  Paris-Saclay, \\
                  Paris, France.
                  \And
                  CMAP, \\
                  \'Ecole Polytechnique, \\
                  Palaiseau, France.
                  \And
                  LPSM, \\
                  Sorbonne Universit\'e, \\
                  Paris, France.
                  \And
                  LTCI, \\
                  T\'el\'ecom ParisTech, \\
                  Universit\'e  Paris-Saclay, \\
                  Paris, France. }
                  ]
\begin{abstract}
Machine learning has witnessed tremendous success in solving tasks depending on
a single hyperparameter. When considering simultaneously a finite number of
tasks, multi-task learning enables one to account for the similarities of the tasks
via appropriate regularizers. A step further consists of learning a continuum
of tasks for various loss functions. A promising approach, called
\emph{\acl{PTL}}, has paved the way in the continuum setting for affine models
and piecewise-linear loss functions.  In this work, we introduce a novel
approach called \emph{\acl{ITL}} whose goal is to learn a function whose output
is a function over the hyperparameter space.  We leverage tools from
operator-valued kernels and the associated \acl{vv-RKHS} that provide an
explicit control over the role of the hyperparameters, and also allows us to
consider new type of constraints. We provide generalization guarantees to the
suggested scheme and illustrate its efficiency in cost-sensitive
classification, quantile regression and density level set estimation.
\end{abstract}
%%%%%%%%%%%%%%%%%%%%%%%%%%%%%%%%%%%%%%%%%%%%%%%%%%%%%%%%%%%%%%%%%%%%%%%%%%%%%%
\section{INTRODUCTION}
\label{section:introduction}
Several fundamental problems in machine learning and statistics can be phrased
as the minimization of a loss function described by a hyperparameter.
The hyperparameter might capture numerous aspects of the problem:
\begin{inparaenum}[(i)]
    \item the tolerance \acs{wrt} outliers as the  $\epsilon$-insensitivity in
    \ac{SVR} \citep{vapnik1997support},
    \item importance of smoothness or sparsity such as the weight of the
    $l_2$-norm in Tikhonov regularization \citep{tikhonov77solution},
    $l_1$-norm in \acs{LASSO} \citep{tibshirani1996regression}, or more general
    structured-sparsity inducing norms \citep{bach12optimization},
    \item \ac{DLSE}, see for example one-class support vector machines
    \ac{OCSVM},
    \item confidence as examplified by \ac{QR}, or
    \item importance of different decisions as implemented by \ac{CSC}.
\end{inparaenum}
In various cases including \ac{QR}, \ac{CSC} or \ac{DLSE}, one is interested in
solving the parameterized task for several hyperparameter values. \acl{MTL}
\citep{evgeniou2004regularized} provides a principled way of benefiting from
the  relationship between similar tasks  while preserving local properties of
the algorithms: $\nu$-property in \ac{DLSE} \citep{glazer2013q} or quantile
property in \ac{QR} \citep{takeuchi2006nonparametric}. \par
A natural extension from the traditional multi-task setting is to provide a
prediction tool being able to  deal with \emph{any} value of the
hyperparameter. In their seminal work, \citep{takeuchi2013parametric} extended
 multi-task learning by considering an infinite number of parametrized tasks
in a framework called \acf{PTL}. Specifically, they  prove that, when focusing
on an affine model for each task, one recovers the task-wise solution for the
whole spectrum of hyperparameters, at the cost of having a model piece-wise
linear in the hyperparameter.\par
In this paper, we also relax the affine model assumption on the tasks as well as
the piecewise-linear assumption on the loss, and take a different angle. We
propose \acf{ITL} within the framework of function-valued function learning to
handle a continuum number of parameterized tasks. For that purpose we leverage
tools from operator-valued kernels and the associated \acf{vv-RKHS}. The idea
is that the output is a function on the hyperparameters---modelled as
scalar-valued \ac{RKHS}---, which provides an explicit control over the role of the
hyperparameters, and also enables us to consider new type of constraints. In
the studied framework each task is described by  a (scalar-valued) \ac{RKHS}
over the input space which is capable of dealing with nonlinearities.
The resulting \acs{ITL} formulation relying on \ac{vv-RKHS} specifically encompasses
existing multi-task approaches including joint quantile regression
\citep{sangnier2016joint} or multi-task variants of density level set
estimation \citep{glazer2013q} by encoding a continuum of tasks.\par
Our \emph{contributions} can be summarized as follows:
\begin{itemize}[labelindent=0cm, leftmargin=*,
                topsep=0cm, partopsep=0cm, parsep=0cm, itemsep=0cm]
    \item We propose ITL, a novel \ac{vv-RKHS}-based scheme to learn a
    continuum of tasks parametrized by a hyperparameter and design new
    regularizers.
    \item We prove excess risk bounds on  ITL and illustrate its efficiency in
    quantile regression, cost-sensitive classification, and density level set
    estimation.
\end{itemize}
The paper is structured as follows. The ITL problem is defined in
\cref{section:infinite_tasks}. In \cref{section:results} we detail how the
resulting learning problem can be tackled in \acp{vv-RKHS}. Excess risk
bounds is the focus of \cref{sec:excess-risk}. Numerical results are presented
in \cref{section:numerical_experiments}. Conclusions are drawn in
\cref{section:conclusion}. Details of proofs are given in the supplement.
%
%%%%%%%%%%%%%%%%%%%%%%%%%%%%%%%%%%%%%%%%%%%%%%%%%%%%%%%%%%%%%%%%%%%%%%%%%%%%%%%
\section{FROM PARAMETERIZED TO INFINITE TASK LEARNING}
\label{section:infinite_tasks}
First, after introducing a few notations, we gradually define our goal by
moving from single parameterized tasks (\cref{sec:single-task})  to \acs{ITL}
(\cref{sec:inf-task}) through multi-task learning (\cref{sec:multi-task}).
\paragraph{Notations:}
$\indicator{S}$ is the indicator function of set $S$. $\sum_{i,j=1}^{n,m}$
reads $\sum_{i=1}^n\sum_{j=1}^m$.  $|x|_+ = \max(x,0)$ denotes positive part.
$\functionspace{\inputspace}{\outputspace}$ stands for the set of $\inputspace
\rightarrow \outputspace$ functions.  Let $\Z$ be Hilbert space and $\L(\Z)$ be
the space of $\Z \to \Z$ bounded linear operators.  Let $K: \inputspace \times
\inputspace \to \L(\Z)$ be an operator-valued kernel, \acs{ie} $\sum_{i,j=1}^n
\inner{z_i, K(x_i,x_j)z_j}_{\mcZ} \geq 0$ for all $n\in\Np$ and $x_1, \ldots,
x_n \in \inputspace$ and $z_1, \ldots, z_n \in \Z$ and $K(x, z)=K(z, x)^*$ for
all $x$, $z\in\inputspace$.  $K$ gives rise to the \acl{vv-RKHS}
$\hypothesisspace_K = \lspan \Set{K(\cdot,x)z |\enskip x \in
\inputspace,\enskip z \in \Z } \subset \functionspace{\inputspace}{\Z}$, where
$\lspan\{\cdot\}$ denotes the closure of the linear span of its argument.  For
futher details on \ac{vv-RKHS} the reader is referred to
\citep{carmeli10vector}.
\subsection{Learning Parameterized Tasks}  \label{sec:single-task}
A \emph{supervised parametrized task} is defined as follows. Let $(X,Y)\in
\inputspace \times \outputspace$ be a random variable with joint distribution
$\probability_{X,Y}$ which is assumed to be fixed but unknown.
Instead we have access to $n$ \ac{iid} observations called training samples:
$\trainingset\defeq((x_i,y_i))_{i=1}^{n} \sim \probability_{X,Y}^{\otimes n}$.
Let $\hyperparameterspace$ be the domain of hyperparameters, and
$\parametrizedcost{\hyperparameter}\colon\outputspace \times \outputspace \to
\reals$ be a loss function  associated to $\hyperparameter \in
\hyperparameterspace$. Let $\hypothesisspace \subset
\functionspace{\inputspace}{\outputspace}$ denote our hypothesis class;
throughout the paper $\hypothesisspace$ is assumed to be a Hilbert space with
inner product $\inner{\cdot, \cdot}_{\hypothesisspace}$.  For a given
$\hyperparameter$, the goal is to estimate a minimizer of the expected risk
\begin{align}\label{equation:true_risk}
    \parametrizedrisk(h) \defeq \expectation_{X,Y} [
    \parametrizedcost{\hyperparameter}(Y, h(X))]
\end{align}
over $\hypothesisspace$, using the training sample $\trainingset$. This task
can be addressed by solving the regularized empirical risk minimization problem
\begin{align}\label{equation:emp_risk_reg}
    \min_{h \in \hypothesisspace} \empiricalrisk^{\hyperparameter}( h) +
     \Omega(h),
\end{align}
where $\parametrizedempiricalrisk(h) \defeq \frac{1}{n}\sum_{i=1}^n
\parametrizedcost{\hyperparameter}(y_i,h(x_i))$ is the empirical risk and
$\regularizer: \hypothesisspace \to \reals$ is a regularizer.
Below we give three examples.
\paragraph{\acl{QR}:}
Assume $\outputspace\subseteq\reals$ and $\hyperparameter \in
\openinterval{0}{1}$. For a given hyperparameter $\hyperparameter$, in \acl{QR}
the goal is to predict the $\hyperparameter$-quantile of the real-valued output
conditional distribution $\probability_{Y|X}$. The task can be tackled
using the pinball loss defined in
\cref{equation:pinball_loss} and illustrated in \cref{figure:pinball} \citep{koenker1978regression}.
\begin{align} \label{equation:pinball_loss}
     \parametrizedcost{\hyperparameter}(y, h(x)) &= \abs{\hyperparameter -
     \indicator{\reals_{-}}(y - h(x))}\abs{y - h(x)}, \\ \regularizer(h) &=
     \tfrac{\lambda}{2}\norm{h}^2_{\hypothesisspace}\condition{$\lambda > 0$.}
     \nonumber
\end{align}
\paragraph{\acl{CSC}:}
Our next example considers binary classification ($\outputspace=\Set{-1,1}$)
where a (possibly) different cost is associated with each class; this task
often arises in medical diagnosis.  The sign
of $h\in\hypothesisspace$ yields the estimated class and in cost-sensitive
classification one takes
\begin{align}
    \parametrizedcost{\theta}(y, h(x)) &= \abs{\tfrac{1}{2}(\theta + 1) -
    \indicator{\Set{-1}}(y)}\abs{1 - yh(x)}_{+}, \\ \regularizer(h) &=
    \tfrac{\lambda}{2}\norm{h}^2_{\hypothesisspace} \condition{$\lambda > 0$.}
    \nonumber
\end{align}
The  $\theta\in\closedinterval{-1}{1}$ hyperparameter captures the trade-off
between the importance of correctly classifying the samples having $-1$ and $+1$ labels.
When $\theta$ is close to $-1$, the obtained $h$ focuses on classifying well
class $-1$, and vice-versa. Typically, it is desirable for a physician to
choose \emph{a posteriori} the value of the hyperparameter at which he wants
to predict. Since this cost can rarely be considered to be fixed, this
motivates the idea to learn one model giving access to all hyperparameter
values.\par
\paragraph{\acl{DLSE}:} Examples of parameterized tasks can also be found in the unsupervised setting.
For instance in outlier detection, the goal is to separate outliers
from inliers. A classical technique to tackle this task is \ac{OCSVM}
\citep{scholkopf2000new}. \ac{OCSVM} has a free parameter
$\hyperparameter\in(0, 1]$, which can be proven to be an upper bound on the
fraction of outliers. When using a Gaussian kernel with a bandwidth tending
towards zero, \acs{OCSVM} consistently estimates density level sets
\citep{vert2006consistency}.  This unsupervised learning problem can be
empirically described by the minimization of a regularized empirical risk
$\empiricalrisk^{\hyperparameter}(h,t) +     \Omega(h)$, solved  \emph{jointly} over
$h\in\hypothesisspace$ \emph{and}
$t\in\reals$ with
\begin{align*}
    \parametrizedcost{\hyperparameter}(t, h(x)) &= -t +
    \frac{1}{\hyperparameter}\abs{t - h(x)}_+, \\ \regularizer(h) &=
    \tfrac{1}{2}\norm{h}^2_{\hypothesisspace}. \nonumber
\end{align*}
\subsection{Solving a finite number of tasks as multi-task learning} \label{sec:multi-task}
In all the  aforementioned problems, one is rarely interested in the choice of
a single hyperparameter value ($\hyperparameter$) and associated risk
$\left(\empiricalrisk^{\hyperparameter}\right)$, but rather in the joint
solution of multiple tasks. The naive approach of solving the different tasks
independently can easily lead to inconsistencies. A principled way of solving
many parameterized tasks has been cast as a multi-task learning problem
\citep{Evgeniou2005} which takes into account the similarities between tasks
and helps providing consistent solutions. Assume that we have $p$ tasks
described by parameters $(\theta_j)_{j=1}^p$. The idea of multi-task learning
is to minimize the sum of the local loss functions
$\empiricalrisk^{\hyperparameter_j}$, \ac{ie}
\begin{align*}
  \argmin_{h} \displaystyle\sum\nolimits_{j=1}^p
  \empiricalrisk^{\hyperparameter_j}(h_j) + \Omega(h),
\end{align*}
where the individual tasks are modelled by the real-valued $h_j$ functions
the overall $\reals^p$-valued model is the vector-valued function
$x\mapsto(h_1(x),\ldots,h_p(x))$, and $\Omega$ is a regularization term.\par
It is instructive to consider two concrete examples:
\begin{itemize}[labelindent=0em,leftmargin=*,topsep=0cm,partopsep=0cm,
                parsep=2mm,itemsep=0cm]
    \item In joint quantile regression one can use the regularizer to
    encourage that the predicted conditional quantile estimates for two
    similar quantile values are similar. This idea forms the basis of the
    approach proposed by \citep{sangnier2016joint} who formulates the joint
    quantile regression problem in a vector-valued Reproducing Kernel Hilbert
    Space with an appropriate decomposable kernel that encodes the links
    between the tasks. The obtained solution shows less quantile curve
    crossings compared to estimators not exploiting the dependencies of the
    tasks as well as an improved accuracy.
    \item A multi-task version of \ac{DLSE} has recently been presented by
    \citep{glazer2013q} with the goal of obtaining nested density level sets
    as $\theta$ grows. Similarly to joint quantile regression, it is crucial to
    take into account the similarities of the tasks in the joint model to
    efficiently solve this problem.
\end{itemize}
\subsection{Towards Infinite Task learning} \label{sec:inf-task}
In the following, we propose a novel framework called Infinite Task Learning in
which we learn a function-valued function $h \in \functionspace{\inputspace}{
\functionspace{\hyperparameterspace}{\outputspace}}$. Our goal is to be able to
solve new tasks after the learning phase and thus, not to be limited to given
predefined values of the hyperparameter. Regarding this goal, our framework
generalizes the \acl{PTL} approach introduced by
\citet{takeuchi2013parametric}, by allowing nonlinear models and relaxing the
hypothesis of piece-wise linearity of the loss function.
Given $\alpha_\theta$ the parameter of a linear model
$h_\theta(x)=\inner{ \alpha_{\theta},x}$ tackling the task $\theta$, the
\ac{PTL} approach relies on parametric programming to alternate between the
minimization of an empirical risk regularized by some inter-task term $\int
\inner{\alpha_{\theta}, D \alpha_{\theta}}\mathrm{d}\theta$ and learning the
metric $D$, which only works in the piecewise-linear loss setting.
Moreover a nice byproduct of this \acs{vv-RKHS} based approach is
that one can benefit from the functional point of view, design new regularizers
and impose various constraints on the whole continuum of tasks, \acs{eg},
\begin{itemize}
    \item The continuity of the $\hyperparameter \mapsto h(x)(\hyperparameter)$
    function is a natural desirable property: for a given input $x$, the
    predictions on similar tasks should also be similar.
    \item Another example is to impose a shape constraint in \ac{QR}:
     the conditional quantile should be increasing \ac{wrt} the
    hyperparameter $\theta$. This requirement can be imposed through a
    functional view of the problem but not from a finite-dimensional view.
    \item In \ac{DLSE}, to get nested level sets, one would want that for all $
    x \in \mcX$, the decision function $\theta \mapsto
    \indicator{\reals_{+}}(h(x)(\hyperparameter) - t(\hyperparameter))$ changes
    its sign only once.
\end{itemize}
To keep the presentation simple, in the sequel we are going to focus on
\ac{ITL} in the  supervised setting; unsupervised tasks can be handled
similarly. \par
Assume that $h$ belongs to some space $\hypothesisspace \subset
\functionspace{\inputspace}{
\functionspace{\hyperparameterspace}{\outputspace}}$ and introduce an
integrated loss function
\begin{align}\label{equation:integrated_cost0}
    \cost(y, h(x)) \defeq \displaystyle\int_{\hyperparameterspace}
    \hcost(\hyperparameter,y,
    h(x)(\hyperparameter))\mathrm{d}\mu(\hyperparameter),
\end{align}
where  the local loss $\hcost \colon \hyperparameterspace \times \outputspace
\times \outputspace \to \reals$ denotes $\hcost_{\hyperparameter}$ seen as a
function of three variables including the hyperparameter and $\mu$ is a
probability measure on $\hyperparameterspace$ which encodes the importance of
the prediction at different hyperparameter values. Without prior information
and for compact $\hyperparameterspace$, one may consider $\mu$ to be uniform.
The true risk reads then
\begin{align}\label{equation:h-objective}
   R(h) &\defeq \expectation_{X,Y} \left[ \cost(Y,
    h(X))\right].
\end{align}
Intuitively, minimizing the expectation of the integral over $\hyperparameter$
in a rich enough space corresponds to searching for a pointwise minimizer $x
\mapsto h^{*}(x)(\theta)$ of the parametrized tasks introduced in
\cref{equation:true_risk} with, for instance, the implicit space constraint
that $\theta \mapsto h^{*}(x)(\theta)$ is a continuous function for each input
$x$.
We show in \cref{proposition:generalized_excess_risk} that this is
precisely the case in \ac{QR}.\par
Interestingly, the empirical counterpart of the true risk minimization can now
be considered with a much richer family of penalty terms than in the finite
dimensional case:
{\small\begin{align}\label{equation:h-objective-empir}
    \min_{h\in \hypothesisspace} \empiricalrisk(h) + \Omega(h), \quad
    \empiricalrisk(h) \defeq \frac{1}{n} \sum\nolimits_{i=1}^n \cost(y_i, h(x_i)).
\end{align}}
Here, $\Omega(h)$ can be a weighted sum of various penalties
\begin{itemize}[labelindent=0em,leftmargin=*,topsep=0cm,partopsep=0cm,
                parsep=2mm,itemsep=0cm]
    \item imposed directly on $(\hyperparameter,x) \mapsto
    h(x)(\hyperparameter)$, or
    \item integrated constraints on either $\hyperparameter \mapsto
    h(x)(\hyperparameter)$ or $x \mapsto h(x)(\hyperparameter)$ such as
    {\small\begin{align*}
        \int_{\inputspace} \Omega_1(
        h(x)(\cdot))\mathrm{d}\probability(x)\,\text{\normalsize
        or}\, \int_{\hyperparameterspace} \Omega_2(
        h(\cdot)(\hyperparameter))\mathrm{d}\mu(\hyperparameter)
    \end{align*}}
    which allow the property enforced by $\Omega_1$ or
    $\Omega_2$ to hold pointwise on $\inputspace$ or $\hyperparameterspace$
    respectively.
\end{itemize}
It is worthwhile to see a concrete example before turning to solutions
questions: in quantile regression, the monotonicity assumption of the
$\hyperparameter \mapsto h(x)(\hyperparameter)$ function can be encoded by
choosing $\Omega_1$ as
\begin{align*}
    \Omega_1(f) =  \lambda_{nc}\int_{\hyperparameterspace}\abs{ -(\partial f)
    (\hyperparameter)}_+ \mathrm{d}\mu(\hyperparameter)
\end{align*}
Many different models ($\hypothesisspace$) could be applied to solve this
problem.  In our work we consider Reproducing Kernel Hilbert Spaces as they
offer a simple and principled way to define regularizers by the
appropriate choice of kernels and exhibit a significant flexibility.
%%%%%%%%%%%%%%%%%%%%%%%%%%%%%%%%%%%%%%%%%%%%%%%%%%%%%%%%%%%%%%%%%%%%%%%%%%%%%%%
\section{SOLVING THE PROBLEM IN \acp{RKHS}}
\label{section:results}
This section is dedicated to solving the \ac{ITL} problem  defined in
\cref{equation:h-objective-empir}. In \cref{sec:V} we focus on the objective
$(\sampledcost)$. The applied \ac{vv-RKHS} model family is detailed in
\cref{sec:H} with various penalty examples followed by representer theorems,
giving rise to computational tractability.
\subsection{Sampled Empirical Risk} \label{sec:V}
In practice solving \cref{equation:h-objective-empir} can be rather challenging
due to the additional integral over $\hyperparameter$.  One might consider
different numerical integration techniques to handle this issue.
We focus here on \ac{QMC} methods\footnote{See \cref{subsection:sampling} of
the supplement for a discussion on other integration techniques.} as they allow
\begin{inparaenum}[(i)]
    \item efficient optimization over \acp{vv-RKHS} which we will use for
    modelling $\hypothesisspace$ (\cref{theorem:representer_supervised}), and
    \item enable us to derive generalization guarantees
    (\cref{proposition:generalization_supervised}).
\end{inparaenum}
Indeed, let
\begin{align}\label{equation:integrated_cost}
   \sampledcost(y, h(x)) \defeq \sum\nolimits_{j=1}^m w_j
   \hcost(\hyperparameter_j,y, h(x)(\hyperparameter_j))
\end{align}
be the \ac{QMC} approximation of \cref{equation:integrated_cost0}. Let $w_j =
m^{-1}F^{-1}(\theta_j)$, and $(\theta_j)_{j=1}^m$ be a sequence with values in
$[0, 1]^d$ such as the Sobol or Halton sequence where $\mu$ is assumed to be
absolutely continuous \acs{wrt} the Lebesgue measure and $F$ is the associated
cdf.  Using this notation and the training samples
$\trainingset=((x_i,y_i))_{i=1}^n$, the empirical risk takes the form
\begin{align}
    \sampledempiricalrisk(h) \defeq \frac{1}{n} \sum\nolimits_{i=1}^n
    \widetilde{\cost}(y_i, h(x_i))
\end{align}
and the problem to solve is
\begin{align}\label{equation:h-objective-empir2}
    \min_{h\in \hypothesisspace} \sampledempiricalrisk(h) + \Omega(h).
\end{align}
\subsection{Hypothesis class ($\hypothesisspace$)} \label{sec:H}
Recall that $\hypothesisspace \subseteq
\functionspace{\inputspace}{\functionspace{
\hyperparameterspace}{\outputspace}}$, in other words $h(x)$ is a
$\hyperparameterspace \mapsto \outputspace$ function for all $x \in
\inputspace$.  In this work we assume that $\outputspace \subseteq \reals$ and
the $\hyperparameterspace \mapsto \outputspace$ mapping can be described by an
\acs{RKHS} $\mcH_{k_{\hyperparameterspace}}$ associated to a
$k_{\hyperparameterspace} \colon \hyperparameterspace \times
\hyperparameterspace \to \reals$ scalar-valued kernel defined on the
hyperparameters.  Let $k_{\mcX} \colon \mcX \times \mcX \to \mathbb{R}$ be a
scalar-valued kernel on the input space. The  $x \input \mapsto
(\text{hyperparameter} \mapsto \text{output})$ relation, \acs{ie} $h\colon
\inputspace \to \mcH_{k_{\hyperparameterspace}}$ is then modelled by the
\acl{vv-RKHS} $\hypothesisspace_K = \lspan \Set{K(\cdot,x)f |\enskip x \in
\mcX,\enskip f \in \hypothesisspace_{k_{\hyperparameterspace}}}$, where the
operator-valued kernel $K$ is defined as $K(x,z)= k_{\mcX}(x,z) I$, and
$I=I_{\mcH_{k_{\hyperparameterspace}}}$ is the identity operator on
$\mcH_{k_{\hyperparameterspace}}$. \par
This so-called decomposable \acl{OVK} has several benefits and gives rise to a
function space with a well-known structure. One can consider elements $h \in
\mcH_K$ as having input space $\mcX$ and output space $\mcH_{k_{\Theta}}$, but
also as functions from $(\mcX \times \Theta)$ to $\reals$. It is indeed known
that there is an isometry between $\mcH_K$ and $\mcH_{k_{\mcX}} \otimes
\mcH_{k_{\Theta}}$, the \ac{RKHS} associated to
the product kernel $k_{\mcX} \otimes k_{\Theta}$. The equivalence between these
views allows a great flexibility and enables one to follow a functional point
of view (to analyse statistical aspects) or to leverage the tensor product
point of view (to design new kind of penalization schemes). Below we detail
various regularizers before focusing on the representer theorems.\par
\begin{itemize}[labelindent=0em,leftmargin=*,topsep=0cm,partopsep=0cm,
                parsep=2mm,itemsep=0cm]
    \item \textbf{Ridge penalty}: For \ac{QR} and \ac{CSC}, a natural
    regularization is the squared  \ac{vv-RKHS} norm
      \begin{align}
        \Omega^{\text{RIDGE}}(h) &= \tfrac{\lambda}{2}\norm{h}_{\mcH_K}^2
        \condition{$\lambda >0.$}
      \end{align}
      This choice is amenable to excess risk analysis
      \seep{proposition:generalization_supervised}.  It can be also seen as the
      counterpart of the classical (multi-task regularization term introduced in
      \citep{sangnier2016joint}, compatible with an infinite number of tasks.
      $\norm{\cdot}_{\mcH_K}^2$ acts by constraining the solution to a ball of a
      finite radius within the \ac{vv-RKHS}, whose shape is controlled by both
      $k_{\mcX}$ and $k_{\Theta}$.
    \item \textbf{$L^{2,1}$-penalty}: For
      \ac{DLSE}, the ridge penalty breaks the asymptotic property of
      estimating the density level sets.  In this case, the natural choice is an
      $L^{2,1}$-\ac{RKHS} mixed regularizer
      \begin{align}
        \Omega^{\text{DLSE}}(h)= \frac{1}{2} \int_{\hyperparameterspace} \norm{
        h(\cdot)(\theta)}_{\mcH_{k_{\mcX}}}^2 \mathrm{d}\mu(\theta)
      \end{align}
      which is an example of a $\hyperparameterspace$-integrated penalty. This
      $\Omega$ choice allows the preservation of the $\theta$-property (see
      \cref{fig:iocsvm_nu_novelty}), in other words that the proportion of the
      outliers is $\theta$.
    \item \textbf{Shape constraints}: Taking
    the example of \ac{QR} it is advantageous to ensure the monotonicity of the
    estimated quantile function.
    Let $\partial_{\hyperparameterspace} h$ denotes the derivative of
    $h(x)(\theta)$ with respect to $\theta$. Then one should solve
    \begin{align*}
        &\argmin_{h \in \mcH_K} \sampledempiricalrisk(h) +
        \Omega^{\text{RIDGE}}(h)  \\
        & \text{\acs{st}} \quad \forall (x,\theta) \in \mcX \times \Theta,
        (\partial_{\hyperparameterspace} h)(x)(\theta) \geq 0.
    \end{align*}
    However, the functional constraint
    prevents a tractable optimization scheme and
     to mitigate this bottleneck, we
    penalize if the derivative of $h$
    \acs{wrt} $\hyperparameter$ is negative:
    {\small\begin{align}\label{equation:non_crossing}
        \Omega_{\text{nc}}(h) \defeq \lambda_{\text{nc}}
        \int_{\inputspace}\int_{\hyperparameterspace}\abs{
        -(\partial_{\hyperparameterspace} h)
        (x)(\theta)}_+d\mu(\hyperparameter)d\probability(x).
    \end{align}}
    When $\probability\defeq\probability_{X}$ this penalization can be
    approximated using the same anchors and weights than the one obtained to
    integrate the loss function
    \begin{align}
        \widetilde{\Omega}_{\text{nc}}(h) =
        \lambda_{\text{nc}}\sum\nolimits_{i,j=1}^{n,m}w_j\abs{
        -(\partial_{\inputspace} h) (x_i)(\theta_j)}_+.
    \end{align}
    Thus, one can modify the overall regularizer in \ac{QR} to be
    \begin{align}\label{eq:whole_reg}
        \Omega (h) &\defeq \Omega^{\text{RIDGE}}(h) +
        \widetilde{\Omega}_{\text{nc}}(h).
    \end{align}
\end{itemize}
\subsection{Representer theorems}
Apart from the flexibility of regularizer design, the other advantage of
applying vv-RKHS as hypothesis class is that it gives rise to
finite-dimensional representation of the ITL solution under mild conditions.
The representer theorem  \cref{theorem:representer_supervised} applies to
\ac{CSC} when $\lambda_{nc}=0$ and to \ac{QR} when $\lambda_{nc} > 0$.
\begin{proposition}[Representer] \label{theorem:representer_supervised}
    Assume that for $\forall \hyperparameter \in \hyperparameterspace,
    \parametrizedcost{\hyperparameter}$ is a proper lower semicontinuous convex
    function with respect to its second argument. Then
    \begin{align*}
        \argmin_{h \in \mcH_{K}} \sampledempiricalrisk(h) + \Omega (h)
        \condition{$\lambda >0$}
    \end{align*}
    with $\Omega(h)$ defined as in \cref{eq:whole_reg}, has a unique solution
    $h^*$, and $\exists$ $\left(\alpha_{ij}\right)_{i,j = 1}^{n,m},
    \left(\beta_{ij}\right)_{i,j = 1}^{n,m} \in \mathbb{R}^{2nm}$ such that
    $\forall x \in \inputspace$
    {\small\begin{dmath*}
        h^*(x) = \sum_{i=1}^{n} k_{\inputspace}(x, x_i)\left(\sum_{j=1}^m
        \alpha_{ij} k_{\hyperparameterspace}(\cdot,\hyperparameter_j) +
        \beta_{ij}(\partial_2
        k_{\hyperparameterspace})(\cdot,\hyperparameter_j)\right).
    \end{dmath*}}
\end{proposition}
\begin{sproof}
    First, we prove that the function to minimize is coercive, convex, lower
    semicontinuous, hence it has a unique minimum. Then $\mcH_K$ is decomposed
    into two orthogonal subspaces and we use the reproducing property to get
    the finite representation.
\end{sproof}
For \ac{DLSE}, we similarly get a representer theorem with the following modelling choice.
The hypothesis space for $h$ is still $\mcH_K$ but parameter $t$ becomes a function over
the hyperparameter space, belonging to $\mcH_{k_{b}}$, the \ac{RKHS} associated
with some scalar kernel $k_{b}:\hyperparameterspace \times \hyperparameterspace
\rightarrow \reals$ that might be different from $k_{\hyperparameterspace}$.
Assume also that $\hyperparameterspace \subseteq \closedinterval{\epsilon}{1}$
where $\epsilon > 0$\footnote{We choose $\hyperparameterspace \subseteq
\closedinterval{\epsilon}{1}$, $\epsilon > 0$ rather than
$\hyperparameterspace\subseteq\closedinterval{0}{1}$ because the loss might not
be integrable on $\closedinterval{0}{1}$.}. Then, learning a
continuum of level sets boils down to the minimization problem
\begin{align}\label{problem:ocsvm}
    &\argmin_{h \in \mcH_K, t \in \mcH_{k_{b}}} \sampledempiricalrisk(h,t) +
    \widetilde{\Omega}(h,t) \condition{$\lambda >0$},
\end{align}
where
\begin{align*}
  \sampledempiricalrisk(h,t) &= \tfrac{1}{n}
  \sum\nolimits_{i, j=1}^{n, m} \tfrac{w_j}{\hyperparameter_j} \left(
  \abs{t(\hyperparameter_j) - h(x_i)(\hyperparameter_j)}_+ -
  t(\hyperparameter_j)\right),  \\
  \widetilde{\Omega}(h,t)
  &= \tfrac{1}{2} \sum\nolimits_{j=1}^m w_j
  \norm{h(\cdot)(\hyperparameter_j)}_{\mcH_{k_{\mcX}}}^2  + \tfrac{\lambda}{2}
  \norm{t}_{\mcH_{k_{b}}}^2.
\end{align*}
\begin{proposition}[Representer] \label{theorem:representer_ocsvm}
    Assume that $k_{\hyperparameterspace}$ is bounded: $\sup_{\hyperparameter
    \in \Theta} k_{\hyperparameterspace}(\hyperparameter,\theta) < +\infty$.
    Then the  minimization problem described in \cref{problem:ocsvm} has a
    unique solution $(\minimizer{h}, \minimizer{t})$ and  there exist
    $\left(\alpha_{ij}\right)_{i,j = 1}^{n,m}\in \reals^{n\times m}$ and $\left
    ( \beta_{j} \right)_{j=1}^m \in \reals^m$ such that for $\forall
    (x,\hyperparameter) \in \inputspace \times \closedinterval{\epsilon}{1}$,
     \begin{align*}
            \minimizer{h}(x)(\hyperparameter) &=
            \sum\nolimits_{i,j=1}^{n,m}
            \alpha_{ij} k_{\mcX}(x,x_i)
            k_{\hyperparameterspace}(\hyperparameter,\hyperparameter_j),\\
            t^*(\hyperparameter) &= \sum\nolimits_{j=1}^{m} \beta_{j}
            k_{b}(\hyperparameter,\hyperparameter_j).
     \end{align*}
\end{proposition}
\begin{sproof}
    First we show that the infimum exists, and that it must be attained in some
    subspace of $\mcH_K \times \mcH_{k_b}$ over which the objective function is
    coercive. By the reproducing property, we get the claimed finite decomposition.
\end{sproof}
\begin{table*}[!tp]
    \caption{Quantile Regression on 20 \acs{UCI} datasets. Reported: $100
    \times$value of the pinball loss, $100 \times $crossing loss (smaller is
    better). \acs{pval}: outcome of the Mann-Whitney-Wilcoxon test of \ac{JQR}
    \acs{vs} $\infty$-{QR} and Independent \acs{vs} $\infty$-\ac{QR}. Boldface:
    significant values.  \label{table:quantile_results}}
    \addtolength{\tabcolsep}{-3pt}
    \renewcommand{\arraystretch}{0.8}% Tighter
    \begin{center}
        \begin{tiny}
            \begin{sc}
                \resizebox{\textwidth}{!}{%
                \begin{tabular}{ccccccccccc}
                    \toprule
                    \multirow{2}{*}{dataset}& \multicolumn{4}{c}{\acs{JQR}} &
                    \multicolumn{4}{c}{IND-\ac{QR}} &
                    \multicolumn{2}{c}{$\infty$-\ac{QR}} \\
                    \cmidrule(lr){2-5}
                    \cmidrule(lr){6-9}
                    \cmidrule(lr){10-11}
                    & (pinball & \acs{pval}) & (cross & \acs{pval}) & (pinball
                    & \acs{pval}) & (cross & \acs{pval}) & pinball & cross \\
                    \midrule
                    CobarOre & $159 \pm 24$ & $9 \cdot 10^{-01}$ & $0.1 \pm
                    0.4$ & $6 \cdot 10^{-01}$ & $150 \pm 21$ & $2 \cdot
                    10^{-01}$ & $0.3 \pm 0.8$ & $7 \cdot 10^{-01}$ & $165 \pm
                    36$ & $2.0 \pm 6.0$ \\
                    engel & $175 \pm 555$ & $6 \cdot 10^{-01}$ & $0.0 \pm 0.2$
                    & $1 \cdot 10^{+00}$ & $63 \pm 53$ & $8 \cdot 10^{-01}$ &
                    $4.0 \pm 12.8$ & $8 \cdot 10^{-01}$ & $47 \pm 6$ & $0.0 \pm
                    0.1$ \\
                    BostonHousing & $49 \pm 4$ & $8 \cdot 10^{-01}$ & $0.7 \pm
                    0.7$ & $2 \cdot 10^{-01}$ & $49 \pm 4$ & $8 \cdot 10^{-01}$
                    & $\mathbf{1.3 \pm 1.2}$ & $1 \cdot 10^{-05}$ & $49 \pm 4$
                    & $0.3 \pm 0.5$ \\
                    caution & $88 \pm 17$ & $6 \cdot 10^{-01}$ & $0.1 \pm 0.2$
                    & $6 \cdot 10^{-01}$ & $89 \pm 19$ & $4 \cdot 10^{-01}$ &
                    $\mathbf{0.3 \pm 0.4}$ & $2 \cdot 10^{-04}$ & $85 \pm 16$ &
                    $0.0 \pm 0.1$ \\
                    ftcollinssnow & $154 \pm 16$ & $8 \cdot 10^{-01}$ & $0.0
                    \pm 0.0$ & $6 \cdot 10^{-01}$ & $155 \pm 13$ & $9 \cdot
                    10^{-01}$ & $0.2 \pm 0.9$ & $8 \cdot 10^{-01}$ & $156 \pm
                    17$ & $0.1 \pm 0.6$ \\
                    highway & $103 \pm 19$ & $4 \cdot 10^{-01}$ & $0.8 \pm 1.4$
                    & $2 \cdot 10^{-02}$ & $99 \pm 20$ & $9 \cdot 10^{-01}$ &
                    $\mathbf{6.2 \pm 4.1}$ & $1 \cdot 10^{-07}$ & $105 \pm 36$
                    & $0.1 \pm 0.4$ \\
                    heights & $127 \pm 3$ & $1 \cdot 10^{+00}$ & $0.0 \pm 0.0$
                    & $1 \cdot 10^{+00}$ & $127 \pm 3$ & $9 \cdot 10^{-01}$ &
                    $0.0 \pm 0.0$ & $1 \cdot 10^{+00}$ & $127 \pm 3$ & $0.0 \pm
                    0.0$ \\
                    sniffer & $43 \pm 6$ & $8 \cdot 10^{-01}$ & $0.1 \pm 0.3$ &
                    $2 \cdot 10^{-01}$ & $44 \pm 5$ & $7 \cdot 10^{-01}$ &
                    $\mathbf{1.4 \pm 1.2}$ & $6 \cdot 10^{-07}$ & $44 \pm 7$ &
                    $0.1 \pm 0.1$ \\
                    snowgeese & $55 \pm 20$ & $7 \cdot 10^{-01}$ & $0.3 \pm
                    0.8$ & $3 \cdot 10^{-01}$ & $53 \pm 18$ & $6 \cdot
                    10^{-01}$ & $0.4 \pm 1.0$ & $5 \cdot 10^{-02}$ & $57 \pm
                    20$ & $0.2 \pm 0.6$ \\
                    ufc & $81 \pm 5$ & $6 \cdot 10^{-01}$ & $\mathbf{0.0 \pm
                    0.0}$ & $4 \cdot 10^{-04}$ & $82 \pm 5$ & $7 \cdot
                    10^{-01}$ & $\mathbf{1.0 \pm 1.4}$ & $2 \cdot 10^{-04}$ &
                    $82 \pm 4$ & $0.1 \pm 0.3$ \\
                    BigMac2003 & $80 \pm 21$ & $7 \cdot 10^{-01}$ &
                    $\mathbf{1.4 \pm 2.1}$ & $4 \cdot 10^{-04}$ & $74 \pm 24$ &
                    $9 \cdot 10^{-02}$ & $\mathbf{0.9 \pm 1.1}$ & $7 \cdot
                    10^{-05}$ & $84 \pm 24$ & $0.2 \pm 0.4$ \\
                    UN3 & $98 \pm 9$ & $8 \cdot 10^{-01}$ & $0.0 \pm 0.0$ & $1
                    \cdot 10^{-01}$ & $99 \pm 9$ & $1 \cdot 10^{+00}$ &
                    $\mathbf{1.2 \pm 1.0}$ & $1 \cdot 10^{-05}$ & $99 \pm 10$ &
                    $0.1 \pm 0.4$ \\
                    birthwt & $141 \pm 13$ & $1 \cdot 10^{+00}$ & $0.0 \pm 0.0$
                    & $6 \cdot 10^{-01}$ & $140 \pm 12$ & $9 \cdot 10^{-01}$ &
                    $0.1 \pm 0.2$ & $7 \cdot 10^{-02}$ & $141 \pm 12$ & $0.0
                    \pm 0.0$ \\
                    crabs & $\mathbf{11 \pm 1}$ & $4 \cdot 10^{-05}$ & $0.0 \pm
                    0.0$ & $8 \cdot 10^{-01}$ & $\mathbf{11 \pm 1}$ & $2 \cdot
                    10^{-04}$ & $\mathbf{0.0 \pm 0.0}$ & $2 \cdot 10^{-05}$ &
                    $13 \pm 3$ & $0.0 \pm 0.0$ \\
                    GAGurine & $61 \pm 7$ & $4 \cdot 10^{-01}$ & $0.0 \pm 0.1$
                    & $3 \cdot 10^{-03}$ & $62 \pm 7$ & $5 \cdot 10^{-01}$ &
                    $\mathbf{0.1 \pm 0.2}$ & $4 \cdot 10^{-04}$ & $62 \pm 7$ &
                    $0.0 \pm 0.0$ \\
                    geyser & $105 \pm 7$ & $9 \cdot 10^{-01}$ & $0.1 \pm 0.3$ &
                    $9 \cdot 10^{-01}$ & $105 \pm 6$ & $9 \cdot 10^{-01}$ &
                    $0.2 \pm 0.3$ & $6 \cdot 10^{-01}$ & $104 \pm 6$ & $0.1 \pm
                    0.2$ \\
                    gilgais & $51 \pm 6$ & $5 \cdot 10^{-01}$ & $0.1 \pm 0.1$ &
                    $1 \cdot 10^{-01}$ & $49 \pm 6$ & $6 \cdot 10^{-01}$ &
                    $\mathbf{1.1 \pm 0.7}$ & $2 \cdot 10^{-05}$ & $49 \pm 7$ &
                    $0.3 \pm 0.3$ \\
                    topo & $69 \pm 18$ & $1 \cdot 10^{+00}$ & $0.1 \pm 0.5$ &
                    $1 \cdot 10^{+00}$ & $71 \pm 20$ & $1 \cdot 10^{+00}$ &
                    $\mathbf{1.7 \pm 1.4}$ & $3 \cdot 10^{-07}$ & $70 \pm 17$ &
                    $0.0 \pm 0.0$ \\
                    mcycle & $66 \pm 9$ & $9 \cdot 10^{-01}$ & $0.2 \pm 0.3$ &
                    $7 \cdot 10^{-03}$ & $66 \pm 8$ & $9 \cdot 10^{-01}$ &
                    $\mathbf{0.3 \pm 0.3}$ & $7 \cdot 10^{-06}$ & $65 \pm 9$ &
                    $0.0 \pm 0.1$ \\
                    cpus & $\mathbf{7 \pm 4}$ & $2 \cdot 10^{-04}$ &
                    $\mathbf{0.7 \pm 1.0}$ & $5 \cdot 10^{-04}$ & $\mathbf{7
                    \pm 5}$ & $3 \cdot 10^{-04}$ & $\mathbf{1.2 \pm 0.8}$ & $6
                    \cdot 10^{-08}$ & $16 \pm 10$ & $0.0 \pm 0.0$ \\
                    \bottomrule
                \end{tabular}}
            \end{sc}
        \end{tiny}
    \end{center}
    \addtolength{\tabcolsep}{3pt}
    \renewcommand{\arraystretch}{1.0}% Tighter
\end{table*}
\textbf{Remarks:}
\begin{itemize}[labelindent=0em,leftmargin=*,topsep=0cm,partopsep=0cm,
                parsep=2mm,itemsep=0cm]
    \item Models with bias: it can be advantageous to add a
    bias to the model, which is here a function of the hyperparameter
    $\hyperparameter$: $\label{equation:biased_models} h(x)(\hyperparameter) =
    f(x)(\hyperparameter) + b(\hyperparameter)$, $f\in\mathcal{H}_K$,
    $b\in\mathcal{H}_{k_b}$, where $k_b:\hyperparameterspace  \times
    \hyperparameterspace \rightarrow \reals$ is a scalar-valued kernel.  This
    can be the case  for example if the kernel on the hyperparameters is the
    constant kernel, \acs{ie} $k_{\hyperparameterspace}(\hyperparameter,
    \hyperparameter')=1$ ($\forall \hyperparameter, \hyperparameter'\in
    \hyperparameterspace$), hence the model $f(x)(\hyperparameter)$ would not
    depend on $\hyperparameter$. An analogous statement to
    \cref{theorem:representer_supervised} still holds for the biased model if
    one adds a regularization $\lambda_b\norm{b}_{\mcH_{k_b}}^2$, $\lambda_b>0$
    to the risk.
    \item Relation to \ac{JQR}: In \ac{IQR}, by choosing
    $k_{\hyperparameterspace}$ to be the Gaussian kernel, $k_b(x, z) =
    \indicator{\Set{x}}(z)$, $\mu = \frac{1}{m}\sum_{j=1}^m \delta_{\theta_j}$,
    where $\delta_{\theta}$ is the Dirac measure concentrated on $\theta$, one
    gets back \citet{sangnier2016joint}'s Joint Quantile Regression (\ac{JQR})
    framework as a special case of our approach. In contrast to the \ac{JQR},
    however, in \ac{IQR} one can predict the quantile value at any
    $\hyperparameter \in (0,1)$, even outside the $(\theta_j)_{j=1}^m$ used for
    learning.
    \item Relation to q-\acs{OCSVM}: In \ac{DLSE},
    by choosing $k_{\hyperparameterspace}(\hyperparameter, \hyperparameter')=1$
    (for all $\hyperparameter\,, \hyperparameter'\in \hyperparameterspace$) to
    be the constant kernel, $k_b(\theta, \theta') =
    \indicator{\Set{\theta}}(\theta')$, $\mu=\frac{1}{m}\sum_{j=1}^m
    \delta_{\theta_j}$, our approach specializes to q-\acs{OCSVM}
    \citep{glazer2013q}.
    \item Relation to \citet{kadri16operator}: Note that \acl{OVK}s for
    functional outputs have also been used in \citep{kadri16operator}, under
    the form of integral operators acting on $L^2$ spaces. Both
    kernels give rise to the same space of functions, the benefit of our
    approach being to provide an \emph{exact} finite representation of the
    solution \seep{theorem:representer_supervised}.
\end{itemize}
%%%%%%%%%%%%%%%%%%%%%%%%%%%%%%%%%%%%%%%%%%%%%%%%%%%%%%%%%%%%%%%%%%%%%%%%%%%%%%%
\section{Excess Risk Bounds}\label{sec:excess-risk}
Below we give generalization error to solution of \cref{equation:h-objective-empir2} for
\ac{QR} and \ac{CSC} (with Ridge regularization and without shape constraints)
by stability
argument \citep{bousquet2002stability}, extending the work of
\citet{Audiffren13} to Infinite-Task Learning. The proposition (finite sample
bounds are given in \cref{corollary:beta_stab_qr}) instantiates the guarantee
for the \ac{QMC} scheme.
\begin{proposition}[Generalization]%
    \label{proposition:generalization_supervised}
    Let $h^* \in \mcH_K$ be the solution of \cref{equation:h-objective-empir2} for
    the \ac{QR} or \ac{CSC} problem with \ac{QMC} approximation. Under mild
    conditions on the kernels $k_{\mcX},k_{\Theta}$ and $\probability_{X,Y}$,
    stated in the supplement, one has
    {\small\begin{align}\label{equation:excess_risk}
        \risk(h^*) \leq \widetilde{R}_{\mcS}(h^*) +
        \mathcal{O}_{\probability_{X,Y}} \left (\frac{1}{\sqrt{\lambda n}}
        \right ) + \mathcal{O} \left ( \frac{\log(m)}{\sqrt{\lambda}m} \right).
    \end{align}}
\end{proposition}
\begin{sproof}
    The error resulting from sampling $\probability_{X, Y}$ and the inexact
    integration is respectively bounded by $\beta$-stability \citep{kadri2015operator} and
    \ac{QMC} results.\footnote{The \ac{QMC} approximation may involve the Sobol
    sequence with discrepancy $m^{-1}\log(m)^s$
    ($s=dim(\hyperparameterspace)$).}
\end{sproof}
\paragraph{$\mathbf{(n,m)}$ trade-off:}
 The proposition reveals the interplay
  between the two approximations, $n$ (the number of training samples) and $m$
  (the number of locations taken in the integral approximation), and allows to
  identify the regime in $\lambda=\lambda(n,m)$ driving the excess risk to
  zero. Indeed by choosing $m=\sqrt{n}$ and discarding logarithmic factors for
  simplicity, $\lambda\gg n^{-1}$ is sufficient.
 The mild assumptions imposed are: boundedness on both
 kernels and the random variable $Y$, as well as some smoothness of the kernels.
%%%%%%%%%%%%%%%%%%%%%%%%%%%%%%%%%%%%%%%%%%%%%%%%%%%%%%%%%%%%%%%%%%%%%%%%%%%%%%%
\section{Numerical Examples}
\label{section:numerical_experiments}
In this section we provide numerical examples illustrating the efficiency of
the proposed ITL approach. We used the following datasets in our experiments: \par
\begin{figure*}[!tp]
    \centering
    \centering\resizebox{!}{.5\linewidth}{%
        \includegraphics{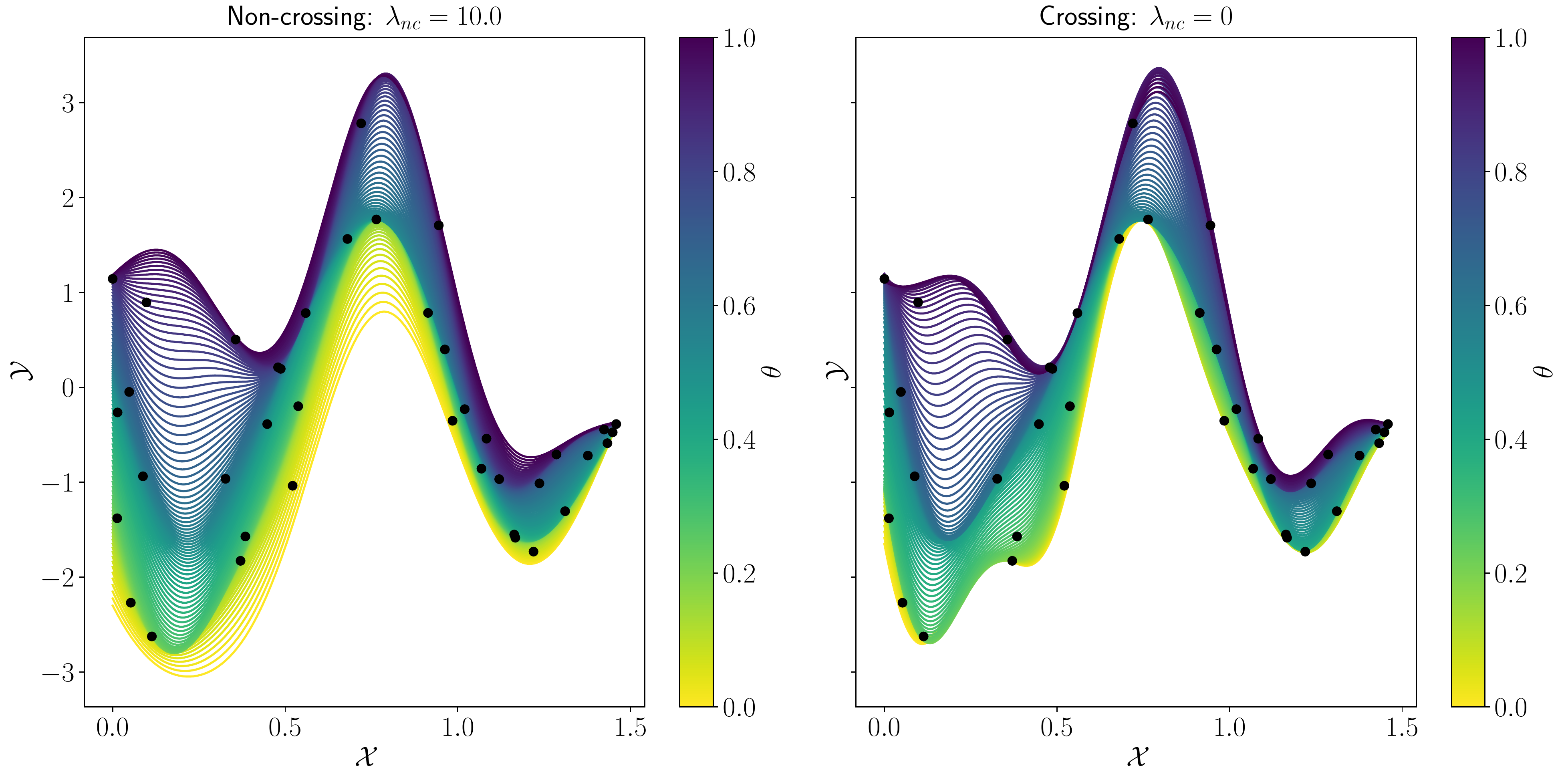}}
    \caption{Impact of crossing penalty on toy data. Left plot: strong
    non-crossing penalty ($\lambda_{\text{nc}}=10$). Right plot: no
    non-crossing penalty ($\lambda_{\text{nc}}=0)$. The plots show $100$
    quantiles of the continuum learned, linearly spaced between $0$ (blue) and
    $1$ (red). Notice that the non-crossing penalty does not provide crossings
    to occur in the regions where there is no points to enforce the penalty
    (\acs{eg} $x\in\closedinterval{0.13}{0.35}$). This phenomenon is
    alleviated by the regularity of the model.  \label{figure:iqr_crossing}}
\end{figure*}
\label{paragraph:datasets}
\begin{itemize}[labelindent=0cm,leftmargin=*,topsep=0cm,partopsep=0cm,
                parsep=2mm,itemsep=0cm]
    \item \acl{QR}: we used
        \begin{inparaenum}[(i)]
            \item a sine synthetic benchmark \citep{sangnier2016joint}: a sine
            curve at $1Hz$ modulated by a sine envelope at $1/3Hz$ and mean
            $1$, distorted with a Gaussian noise of mean 0 and a linearly
            decreasing standard deviation from $1.2$ at $x=0$ to $0.2$ at
            $x=1.5$.
            \item $20$ standard regression datasets from \acs{UCI}.
            The number of samples varied between $38$ (CobarOre) and $1375$
            (Height).  The observations were standardised to have unit
            variance and zero mean for each attribute.
        \end{inparaenum}
    \item \acl{CSC}: The Iris \acs{UCI} dataset with $4$ attributes and $150$
    samples. The two synthetic \textsc{scikit-learn}
    \citep{pedregosa2011scikit} datasets \textsc{Two-Moons} (noise=$0.4$) and
    $\textsc{Circles}$ (noise=$0.1$) with both $2$ attributes and $1000$
    samples. A third synthetic \textsc{scikit-learn} dataset \textsc{Toy}
    (class sep=$0.5$) with $20$ features ($4$ redundant and $10$ informative)
    and $n=1000$ samples.
    \item \acl{DLSE}: The Wilt database from the \acs{UCI} repository with
    $4839$ samples and $5$ attributes, and the Spambase \acs{UCI} dataset with
    $4601$ samples and $57$ attributes served as benchmarks.
\end{itemize}
\paragraph{Note on Optimization:}
There are several ways to solve the non-smooth optimization problems  associated
to the \ac{QR}, \ac{DLSE} and \ac{CSC} tasks. One could proceed for example by
duality---as it was done in JQR \citet{sangnier2016joint}---, or apply
sub-gradient descent techniques (which often converge quite slowly). In order
to allow unified treatment and efficient solution
in our experiments we used the \acs{L-BFGS-B} \citep{zhu1997algorithm}
optimization scheme which is widely popular in large-scale learning, with
non-smooth extensions \citep{skajaa2010limited,keskar2017limited}. The
technique requires only evaluation of objective function along with its
gradient, which can be computed automatically using reverse mode automatic
differentiation (as in \citet{abadi2016tensorflow}). To benefit from  from the
available fast smooth implementations \citep{jones2014scipy,fei2014parallel},
we applied an infimal convolution
(\seet{subsection:infimal_convo} of the supplementary material) on the
non-differentiable terms of the objective. Under the assumtion that
$m=\mathcal{O}(\sqrt{n})$ (see \cref{proposition:generalization_supervised}),
the complexity per \ac{L-BFGS-B} iteration is $\mathcal{O}(n^2\sqrt{n})$.
An experiment showing the impact of increasing $m$ on a synthetic dataset is
provided in the supplement (\cref{figure:iqr_m}).
The Python library replicating our experiments is available in the supplement.\par
\paragraph{\ac{QR}:}
The efficiency of the non-crossing penalty is illustrated in
\cref{figure:iqr_crossing} on the synthetic sine wave dataset described in
\cref{paragraph:datasets} where $n=40$ and $m=20$ points have been generated.
Many crossings are visible on the right plot, while they are almost not
noticible on the left plot, using the non-crossing penalty.
\begin{table*}[!htbp]
    \caption{\acs{ICSC} vs Independent (IND)-\acs{CSC}. Higher is
    better.\label{table:csc_results}}
    \addtolength{\tabcolsep}{-3pt}
    \renewcommand{\arraystretch}{0.8}% Tighter
    \begin{center}
        \begin{tiny}
            \begin{sc}
                \resizebox{.7\textwidth}{!}{%
                \begin{tabular}{cccccccc}
                    \toprule
                    \multirow{2}{*}{Dataset} & \multirow{2}{*}{Method} &
                    \multicolumn{2}{c}{$\theta=-0.9$} &
                    \multicolumn{2}{c}{$\theta=0$} &
                    \multicolumn{2}{c}{$\theta=+0.9$} \\
                    \cmidrule(lr){3-4} \cmidrule(lr){5-6} \cmidrule(lr){7-8} &
                    & \textsc{sensitivity} & \textsc{specificity} &
                    \textsc{sensitivity} & \textsc{specificity} &
                    \textsc{sensitivity} & \textsc{specificity} \\
                    \midrule
                    \multirow{ 2}{*}{\textsc{Two-Moons}} & IND &
                    $0.3\pm0.05$ & $0.99\pm0.01$ & $0.83\pm0.03$ &
                    $0.86\pm0.03$ & $0.99\pm0$ & $0.32\pm0.06$ \\
                                                & \acs{ICSC} & $0.32\pm0.05$ &
                    $0.99\pm0.01$ & $0.84\pm0.03$ & $0.87\pm0.03$ & $1\pm0$ &
                    $0.36\pm0.04$  \\
                    \multirow{ 2}{*}{\textsc{Circles}} & IND & $0\pm0$
                    & $1\pm0$ & $0.82\pm0.02$ & $0.84\pm0.03$ & $1\pm0$ &
                    $0\pm0$ \\
                                              & \acs{ICSC} & $0.15\pm0.05$ &
                    $1\pm0$ & $0.82\pm0.02$ & $0.84\pm0.03$ & $1\pm0$ &
                    $0.12\pm0.05$  \\
                    \multirow{ 2}{*}{\textsc{Iris}} & IND &
                    $0.88\pm0.08$ & $ 0.94\pm0.06$ & $0.94\pm0.05$ &
                    $0.92\pm0.06$ & $0.97\pm0.05$ & $0.87\pm0.06$\\
                                           & \acs{ICSC} & $0.89\pm0.08$ &
                    $0.94\pm0.05$ & $0.94\pm0.06$ & $0.92\pm0.05$ &
                    $0.97\pm0.04$ & $0.90\pm0.05$
                    \\
                    \multirow{ 2}{*}{\textsc{Toy}} & IND &
                    $0.51\pm0.06$ & $ 0.98\pm0.01$ & $0.83\pm0.03$ &
                    $0.86\pm0.03$ & $0.97\pm0.01$ & $0.49\pm0.07$\\
                                           & \acs{ICSC} & $0.63\pm0.04$ &
                    $0.96\pm0.01$ & $0.83\pm0.03$ & $0.85\pm0.03$ &
                    $0.95\pm0.02$ & $0.61\pm0.04$
                    \\
                    \bottomrule
                \end{tabular}}
            \end{sc}
        \end{tiny}
    \end{center}
    \addtolength{\tabcolsep}{3pt}
    \renewcommand{\arraystretch}{1.0}
\end{table*}
Concerning our real-world examples, to study the efficiency of the proposed
scheme in quantile regression the
following experimental protocol was applied. Each dataset
(\cref{paragraph:datasets}) was splitted randomly into a training set (70\%)
and a test set (30\%). We optimized the hyperparameters by minimizing a
$5$-folds cross validation with a Bayesian optimizer\footnote{We used a
Gaussian Process model and minimized the Expected improvement. The optimizer
was initialized using $27$ samples from a Sobol sequence and ran for $50$
iterations.} (For further details \seet{subsection:proto_exp}).
Once the hyperparameters were obtained, a new regressor was
learned on the whole training set using the optimized hyperparameters. We
report the value of the pinball loss and the crossing loss on the test set for
three methods: our technique is called $\infty$-\ac{QR}, we refer to
\citet{sangnier2016joint}'s approach as  \ac{JQR}, and independent learning
(abbreviated as IND-\ac{QR}) represents a further baseline. \par
We repeated $20$ simulations (different random training-test splits); the
results are also compared using a Mann-Whitney-Wilcoxon test. A summary is
provided in \cref{table:quantile_results}.
Notice that while \ac{JQR} is taylored to predict finite many quantiles, our
$\infty$-\ac{QR} method estimates the \emph{whole quantile function} hence
solves a more challenging task.  Despite the more difficult problem solved, as
\cref{table:quantile_results} suggest that the performance in terms of pinball
loss of $\infty$-\ac{QR} is comparable to that of the state-of-the-art JQR on
all the twenty studied benchmarks, except for the `crabs' and `cpus' datasets
(\acs{pval} $<0.25\%$). In addition, when considering the non-crossing penalty
one can observe that $\infty$-\ac{QR} outperforms the IND-\ac{QR} baseline on
eleven datasets (\acs{pval} $<0.25\%$) and \ac{JQR} on two datasets. This
illustrates the efficiency of the constraint based on the continuum scheme.
\paragraph{\ac{DLSE}:}
To assess the quality of the estimated model by $\infty$-\acs{OCSVM}, we
illustrate the $\theta$-property \citep{scholkopf2000new}: the proportion of
inliers has to be approximately $1-\hyperparameter$ ($\forall \hyperparameter
\in (0,1)$).  For the studied datasets (Wilt, Spambase) we used the raw inputs
without applying any preprocessing.  Our input kernel was the exponentiated
$\chi^2$ kernel $k_{\inputspace}(x, z)\defeq \exp\left(-\gamma_{\inputspace}
\sum_{k=1}^d(x_k - z_k)^2/(x_k + z_k)\right)$ with bandwidth
$\gamma_{\inputspace}=0.25$.  A Gauss-Legendre quadrature rule provided the
integral approximation in \cref{equation:integrated_cost}, with $m=100$
samples. We chose the Gaussian kernel for $k_{\hyperparameterspace}$; its
bandwidth parameter $\gamma_{\hyperparameterspace}$ was the $0.2-$quantile of
the pairwise Euclidean distances between the $\theta_j$'s obtained via the
quadrature rule.  The margin (bias) kernel was $k_b=k_{\hyperparameterspace}$.
As it can be seen in Fig.~\ref{fig:iocsvm_nu_novelty}, the $\theta$-property
holds for the estimate which illustrates the efficiency of the proposed
continuum approach for density level-set estimation.
\begin{figure}[htbp]
    \centering
    \includegraphics[width=\textwidth]{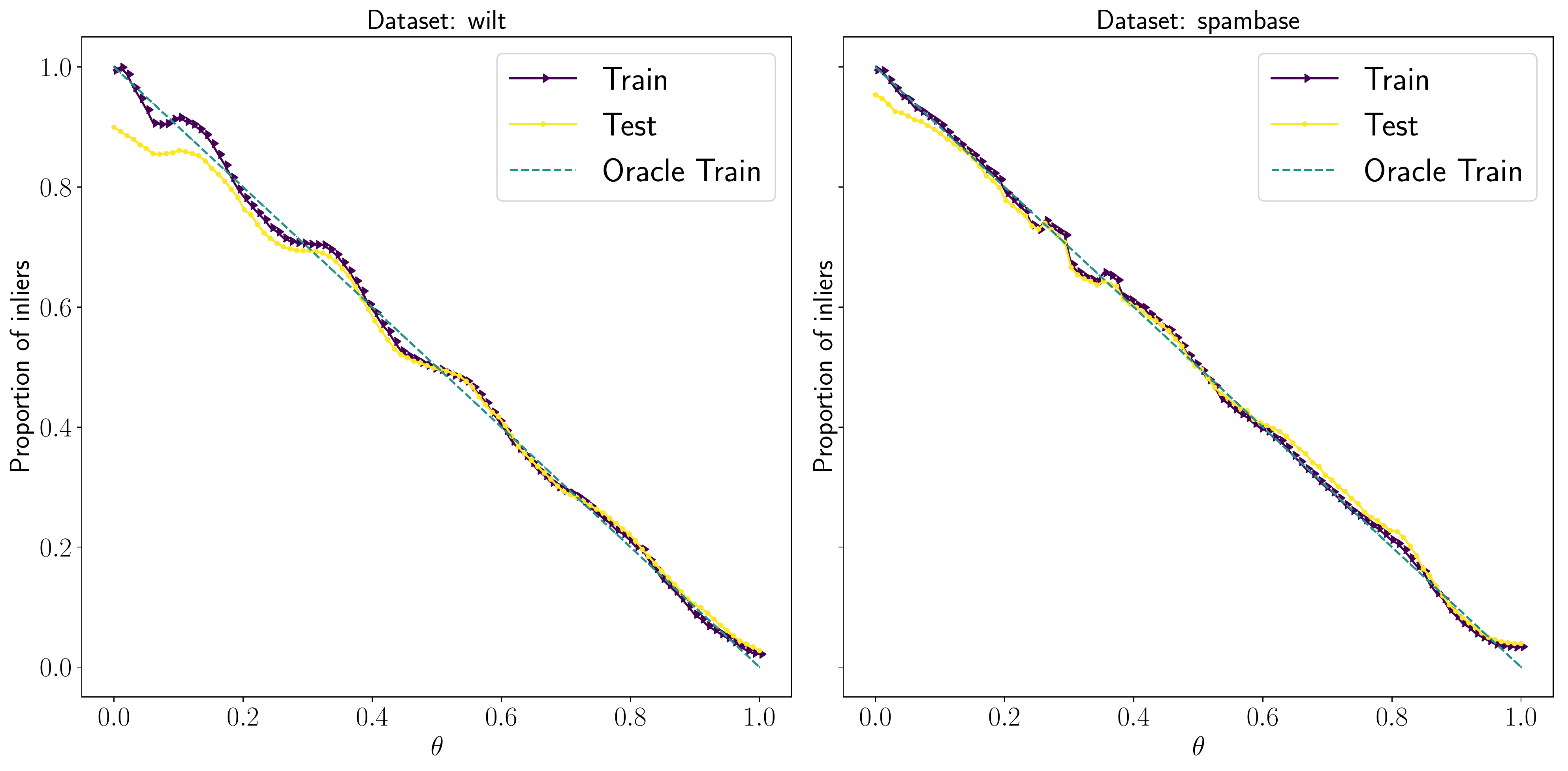}
    \caption{\acl{DLSE}: the $\theta$-property is
    approximately satisfied. \label{fig:iocsvm_nu_novelty}}
\end{figure}
\paragraph{\acl{CSC}:}
As detailed in \cref{section:infinite_tasks}, \acl{CSC} on a continuum
$\hyperparameterspace = \closedinterval{-1}{1}$ that we call \ac{ICSC}
can be tackled by our proposed technique.  In this case, the hyperparameter
$\hyperparameter$ controls the tradeoff between the importance of the correct
classification with labels $-1$ and $+1$. When $\hyperparameter = -1$,
class $-1$ is emphasized; the probability of correctly classified instances
with this label (called specificity) is desired to be $1$.  Similarly, for
$\hyperparameter = +1$, the probability of correct classification of samples
with label $+1$ (called sensitivity) is ideally $1$.\par
To illustrate the advantage of (infinite) joint learning we used two synthetic
datasets \textsc{Circles} and \textsc{Two-Moons} and the \acs{UCI}
\textsc{Iris} dataset. We chose $k_{\inputspace}$ to be a Gaussian kernel with
bandwidth $\sigma_{\inputspace}=(2\gamma_{\inputspace})^{(-1/2)}$ the median of
the Euclidean pairwise distances of the input points \citep{jaakkola1999using}.
$k_{\hyperparameterspace}$ is also a Gaussian kernel with bandwidth
$\gamma_{\hyperparameterspace}=5$.  We used $m=20$ for all datasets.
As a baseline we trained independently 3 \acl{CSC} classifiers with
$\hyperparameter\in\Set{-0.9,0, 0.9}$. We repeated $50$ times a random $50-50\%$
train-test split of the dataset and report the average test error and standard
deviation (in terms of sensitivity and specificity) \par
Our results are illustrated in \cref{table:csc_results}. For $\theta=-0.9$, both
 independent and joint learners give the desired $100\%$ specificity; the
joint \acl{CSC} scheme however has significantly higher sensitivity value
($15\%$ vs $0\%$) on the dataset \textsc{Circles}. Similar conclusion holds for the
$\theta=+0.9$ extreme: the ideal sensitivity is reached by both techniques, but
the joint learning scheme performs better in terms of specificity ($0\%$ vs
$12\%$) on the dataset \textsc{Circles}.
%%%%%%%%%%%%%%%%%%%%%%%%%%%%%%%%%%%%%%%%%%%%%%%%%%%%%%%%%%%%%%%%%%%%%%%%%%%%%%%
\section{Conclusion}
\label{section:conclusion}
In this work we proposed Infinite Task Learning, a novel nonparametric framework
aiming at jointly solving parametrized tasks for a continuum of
hyperparameters.
Future works should study whether local properties of the algorithm
($\theta$-property in \ac{OCSVM}, quantile property in \ac{QR}) are
asymptotically kept true for all hyperparameter values, as well as investigate
acceleration schemes based on kernel approximations.
%%%%%%%%%%%%%%%%%%%%%%%%%%%%%%%%%%%%%%%%%%%%%%%%%%%%%%%%%%%%%%%%%%%%%%%%%%%%%%%
%
\section*{Acknowledgments}
The authors thank Arthur Tenenhaus for some insightful discussions. This work has
been supported by the labex \emph{DigiCosme} as well as the industrial chair \emph{
Machine Learning for Big Data} from T\'el\'ecom ParisTech.
\printbibliography
\clearpage
%%%%%%%%%%%%%%%%%%%%%%%%%%%%%%%%%%%%%%%%%%%%%%%%%%%%%%%%%%%%%%%%%%%%%%%%%%%%%%%
%%%%%%%%%%%%%%%%%%%%%%%%%%%%%%%%%%%%%%%%%%%%%%%%%%%%%%%%%%%%%%%%%%%%%%%%%%%%%%%
% \include{appendix}
\onecolumn
\part*{SUPPLEMENTARY MATERIAL}
\renewcommand{\thesection}{S.\arabic{section}}
\renewcommand{\thesubsection}{S.\arabic{section}.\arabic{subsection}}
\renewcommand{\thesubsubsection}{S.\arabic{section}.\arabic{subsection}%
    .\arabic{subsubsection}}
\renewcommand{\thetable}{S.\arabic{table}}
\renewcommand{\thefigure}{S.\arabic{figure}}
%
%%%%%%%%%%%%%%%%%%%%%%%%%%%%%%%%%%%%%%%%%%%%%%%%%%%%%%%%%%%%%%%%%%%%%%%%%%%%%%%
{\centering
\subsubsection*{Acronyms}
\printacronyms[heading=none]\par}
Below we provide the proofs of the results stated in the main part of the
paper.
\begin{refsection}
\section{Quantile Regression}
\label{appendix:qr}
We remind the expression of the pinball loss \seep{figure:pinball}:
\begin{dmath}[compact]
    v_{\hyperparameter}: (y,\,y') \in
    \reals^2 \mapsto  \max{(\hyperparameter (y-y'), (\hyperparameter
    -1)(y-y'))} \in \reals.
\end{dmath}
\begin{floatingfigure}{.7\textwidth}
    \centering
    \begin{tikzpicture}
        \begin{axis}[%
            axis x line = center,
            axis y line = center,
            disabledatascaling,
            axis equal image,
            xlabel = {$y - h(x)$},
            ylabel = {$v_{\theta}(y, h(x))$},
            x label style={at={(axis description cs:1.0, 0.0)}, anchor=north},
            ticks = none
            ]
            \addplot[orange, domain=-1:1, samples at={-1, 0, 1}]
                {abs(.8 - (x < 0)) * abs(x)};
            \draw
                (-0.5, 0.1) coordinate (a)
                (0, 0) coordinate (b)
                (-0.5, 0) coordinate (c);
            \draw
                pic["$\theta - 1$", draw=black, -,dashed, angle eccentricity=1.2,
                    angle radius=2.3cm]
                {angle=a--b--c};
            \draw
                (0.5, 0.4) coordinate (d)
                (0, 0) coordinate (e)
                (0.5, 0) coordinate (f);
            \draw
                pic["$\theta$", draw=black, -,dashed, angle eccentricity=1.2,
                    angle radius=2.3cm]
                {angle=f--e--d};
        \end{axis}
    \end{tikzpicture}
    \caption{\label{figure:pinball} Pinball loss for $\hyperparameter=0.8$.}
\end{floatingfigure}
\begin{proposition}\label{proposition:generalized_excess_risk}
    Let $X,Y$ be two \acp{rv} respectively taking values in $\mcX$ and
    $\reals$, and $q \colon \mcX \to \mathcal{F}([0,1],\mathbb{R})$ the
    associated conditional quantile function.  Let $\mu$ be a positive measure
    on $[0,1]$ such that $ \int_{0}^1 \expectation\left[
    v_{\hyperparameter}\left(Y,q(X)(\hyperparameter)\right)\right] \mathrm{d}
    \mu(\hyperparameter) < \infty$.  Then for $\forall h \in
    \functionspace{\mathcal{X}}{\functionspace{\closedinterval{0}{1}}{\reals}}$
    \begin{dmath*}
        \risk(h) - \risk(q) \geq 0,
    \end{dmath*}
    where $R$ is the risk defined in \cref{equation:h-objective}.
\end{proposition}
\begin{proof}
    The proof is based on the one given in \citep{li2007quantile} for a single
    quantile. Let $f \in
    \functionspace{\mcX}{\functionspace{\closedinterval{0}{1}}{\reals}}$,
    $\hyperparameter \in (0,1)$ and $(x,y) \in \mcX \times \reals$. Let also
    \begin{align*}
        s &=
        \begin{cases}
            1 ~ \text{if } y \leq f(x)(\hyperparameter)      \\
            0  \text{ otherwise}
        \end{cases},&
        t &=
        \begin{cases}
            1 ~ \text{if } y \leq q(x)(\hyperparameter)      \\
            0  \text{ otherwise}
        \end{cases}.
    \end{align*}
    It holds that
    \begin{dmath*}
        v_{\hyperparameter}(y,h(x)(\hyperparameter)) -
        v_{\hyperparameter}(y,q(x)(\hyperparameter))
        = \hyperparameter(1-s)(y-h(x)(\hyperparameter)) + (\hyperparameter -
        1)s(y-h(x)(\hyperparameter)) -
        \hyperparameter(1-t)(y -q(x)(\hyperparameter)) -
        (\hyperparameter-1)t(y-q(x)(\hyperparameter))
        = \hyperparameter(1-t)(q(x)(\hyperparameter) - h(x)(\hyperparameter)) +
        \hyperparameter((1-t)-(1-s))h(x)(\hyperparameter) +
        (\hyperparameter-1)t(q(x)(\hyperparameter - h(x)(\hyperparameter))) +
        (\hyperparameter-1)(t-s)h(x)(\hyperparameter) + (t-s)y\nonumber
        =     (\hyperparameter -t)(q(x)(\hyperparameter) -
        h(x)(\hyperparameter)) +
        (t-s)(y-h(x)(\hyperparameter)).\label{eq:decompo_pinball}
    \end{dmath*}
    Then, notice that
    \begin{dmath*}[compact]
        \expectation{[(\hyperparameter - t)(q(X)(\hyperparameter) -
        h(X)(\hyperparameter))]} =
        \expectation{[\expectation{[(\hyperparameter - t)(q(X)(\hyperparameter)
        - h(X)(\hyperparameter))]} | X]} =
        \expectation{[\expectation{[(\hyperparameter - t) | X ]}
        (q(X)(\hyperparameter) - h(X)(\hyperparameter))]}
    \end{dmath*}
    and since $q$ is the true quantile function,
    \begin{align*}
        \expectation{ [t | X]} = \expectation{ [\mathbf{1}_{\{ Y \leq
        q(X)(\hyperparameter)\}} | X]} = \probability{[Y \leq
        q(X)(\hyperparameter) | X]} = \hyperparameter,
    \end{align*}
    so
    \begin{align*}
        \expectation{[(\hyperparameter - t)(q(X)(\hyperparameter) -
        h(X)(\hyperparameter))]} =0.
    \end{align*}
    Moreover, $(t-s)$ is negative when $q(x)(\hyperparameter) \leq y \leq
    h(x)(\hyperparameter)$, positive when $h(x)(\hyperparameter) \leq y \leq
    q(x)(\hyperparameter)$ and $0$ otherwise, thus the quantity
    $(t-s)(y-h(x)(\hyperparameter))$ is always positive. As a consequence,
    \begin{dmath*}[compact]
        R(h) - R(q) = \int_{[0,1]}
        \expectation{[v_{\hyperparameter}(Y,h(X)(\hyperparameter)) -
        v_{\hyperparameter}(Y,q(X)(\hyperparameter))]} \mathrm{d}
        \mu(\hyperparameter) \geq 0
    \end{dmath*}
    which concludes the proof.
\end{proof}
The \cref{proposition:generalized_excess_risk} allows us to derive conditions
under which the minimization of the risk above yields the true quantile
function. Under the assumption that (i) $q$ is continuous (as seen as a
function of two variables), (ii) $\mathrm{Supp}(\mu) = [0,1] $, then the
minimization of the integrated pinball loss performed in the space of
continuous functions yields the true quantile function on the support of
$\probability_{X,Y}$.
%%%%%%%%%%%%%%%%%%%%%%%%%%%%%%%%%%%%%%%%%%%%%%%%%%%%%%%%%%%%%%%%%%%%%%%%%%%%%%%
\section{Representer Propositions}
\label{appendix:representer}
\begin{proof} [Proof of \cref{theorem:representer_supervised}]
    First notice that
    \begin{align}
        J: h \in \mcH_K \mapsto \frac{1}{n} \displaystyle\sum_{i=1}^n
        \displaystyle\sum_{j=1}^m w_j\hcost(\hyperparameter_j,y_i,
        h(x_i)(\hyperparameter_j)) + \frac{\lambda}{2} \norm{h}_{\mcH_K}^2 \in
        \reals
    \end{align}
    is a proper lower semicontinuous strictly convex function
    \citep[Corollary 9.4]{bauschke2011convex}, hence $J$ admits a unique
    minimizer $h^* \in \mcH_K$ \citep[Corollary 11.17]{bauschke2011convex}.
    Let
    \begin{dmath}[compact] \label{decompo_theorem}
        \mcU = \sspan{ \Set{
        (K(\cdot,x_i)k_{\Theta}(\cdot,\theta_j))_{i,j=1}^{n,m} | \forall  x_i
        \in \inputspace , \forall \theta_j \in \hyperparameterspace }} \subset
        \mcH_K.
    \end{dmath}
    Then $\mcU$ is a finite-dimensional subspace of $\mcH_K$, thus closed in
    $\mcH_K$, and it holds that $\mcU \oplus \mcU^{\perp} = \mcH_K$, so $h^*$
    can be decomposed as $h^* = h_{\mcU}^* + h_{\mcU^{\perp}}^*$ with
    $h_{\mcU}^* \in \mcU$ and $h_{\mcU^{\perp}}^* \in \mcU^{\perp}$. Moreover,
    for all $1 \leq i \leq n$ and $1 \leq j \leq m$,
    \begin{align*}
        h_{\mcU^{\perp}}^*(x_i)(\theta_j) &= \langle h_{\mcU^{\perp}}^*(x_i) ,
        k_{\Theta}(\cdot,\theta_j) \rangle_{\mcH_{k_{\Theta}}}
        = \langle h_{\mcU^{\perp}}^* , K(\cdot,x_i)k_{\Theta}(\cdot,\theta_j)
        \rangle_{\mcH_K}
        = 0,
    \end{align*}
    so $J(h^*) = J(h_{\mcU}^*) + \lambda \norm{h_{\mcU^{\perp}}^*}_{\mcH_K}^2$.
    However $h^*$ is the minimizer of J, therefore $h_{\mcU^{\perp}}^*=0$ and
    there exist $\left(\alpha_{ij}\right)_{i,j = 1}^{n,m}$ such that $\forall
    x,\hyperparameter \in \mcX \times \hyperparameterspace$,
    $h^*(x)(\hyperparameter) = \sum_{i,j=1}^{n,m} \alpha_{ij} k_{\mcX}(x,x_i)
    k_{\hyperparameterspace}(\hyperparameter,\hyperparameter_j)$.
    \paragraph{Derivative shapes constraints:}
    Reminder: for a function $h$ of one variable, we note $\partial h$ the
    derivative of $h$. For a function $k(\theta, \theta')$ of two variables we
    note $\partial_1 k$ the derivative of $k$ with respect to $\theta$ and
    $\partial_2 k$ the derivative of $k$ with respect to $\theta'$.
    From \citet{zhou2008derivative}, notice that if $f\in\mathcal{H}_k$, where
    $\mathcal{H}_k$ is a scalar-valued \ac{RKHS} on a compact subset
    $\hyperparameterspace$ of $\reals^d$, and $k\in
    \mathcal{C}^{2}(\hyperparameterspace \times \hyperparameterspace)$ (in the
    sense of \citet{ziemer2012weakly}) then $\partial f\in\mathcal{H}_k$. Hence
    if one add a new term of the form:
    \begin{dmath*}
        \lambda_{\text{nc}} \sum_{i=1}^n\sum_{j=1}^m \Omega_{\text{nc}}
        \left(\left(\partial\left[h(x_i)\right]\right)(\theta_j)\right) =
        \lambda_{\text{nc}}\sum_{i=1}^n\sum_{j=1}^m
        \Omega_{\text{nc}}\left((\partial h(x_i))(\theta_j)\right)
    \end{dmath*}
    where $g$ is a strictly monotonically increasing function and
    $\lambda_{\text{nc}} > 0$, a new representer theorem can be obtain by
    constructing the new set
    \begin{dmath*}[compact] \label{decompo_theorem}
        \mcU = \sspan{ \Set{
        (K(\cdot,x_i)k_{\Theta}(\cdot,\theta_j))_{i,j=1}^{n,m} | \forall  x_i
        \in \inputspace , \forall \theta_j \in \hyperparameterspace }
        \cup \Set{
        (K(\cdot,x_i)(\partial_2 k_{\Theta})(\cdot,\theta_j))_{i,j=1}^{n,m} |
        \forall  x_i \in \inputspace , \forall \theta_j \in
        \hyperparameterspace }
        }
        \subset
        \mcH_K.
    \end{dmath*}
    The proof is the same than \cref{theorem:representer_supervised} with the
    new set $\mathcal{U}$ to obtain the expansion $h(x)(\theta) =
    \sum_{i=1}^n\sum_{j=1}^m \alpha_{ij} k_{\inputspace}(x,
    x_i)k_{\hyperparameterspace}(\theta, \theta_j) + \beta_{ij}k(x,
    x_i)(\partial_2k_{\hyperparameterspace})(\theta, \theta_j)$.  For the
    regularization notice that for a symmetric function $(\partial_1 k)(\theta,
    \theta') = (\partial_2 k)(\theta', \theta)$. Hence $\inner{(\partial_1
    k)(\cdot, \theta'), k(\cdot, \theta)}_{\mathcal{H}_k} = \inner{k(\cdot,
    \theta'), (\partial_2 k)(\cdot, \theta)}_{\mathcal{H}_k}$ and $(\partial
    k_{\theta'})(\theta) = (\adjoint{\partial} k_{\theta})(\theta')$ and
    \begin{dmath*}
        \norm{h}_{\mathcal{H}_K}^2
        = \inner{h, h}_{\mathcal{H}_K}
        = \sum_{i=1}^n\sum_{j=1}^m\sum_{i'=1}^n\sum_{j'=1}^m
        \alpha_{ij}\alpha_{i'j'}k_{\inputspace}(x_i,
        x_{i'})k_{\hyperparameterspace}(\theta_j, \theta_{j'}) +
        \alpha_{ij}\beta_{i'j'}k_{\inputspace}(x_i, x_{i'})(\partial_2
        k_{\hyperparameterspace})(\theta_j, \theta_{j'}) +
        \alpha_{i'j'}\beta_{ij}k_{\inputspace}(x_i, x_{i'})(\partial_1
        k_{\hyperparameterspace})(\theta_j, \theta_{j'}) +
        \beta_{ij}\beta_{i'j'}k_{\inputspace}(x_i,
        x_{i'})(\partial_1\partial_2k_{\hyperparameterspace})(\theta_j,
        \theta_{j'})
    \end{dmath*}
    Eventually $(\partial h(x))(\theta) = \sum_{i=1}^n\sum_{j=1}^m \alpha_{ij}
    k_{\inputspace}(x, x_i)(\partial_1 k_{\hyperparameterspace})(\theta,
    \theta_j) + \beta_{ij}k(x,
    x_i)(\partial_1\partial_2k_{\hyperparameterspace})(\theta, \theta_j)$.
\end{proof}
To prove \cref{theorem:representer_ocsvm}, the following lemmas are useful.
\begin{lemma} \citep{carmeli10vector} \label{lemma:isometry}
    Let $k_{\mcX}:\inputspace \times \inputspace \rightarrow \reals$,
    $k_{\hyperparameterspace}:\hyperparameterspace \times \hyperparameterspace
    \rightarrow \reals$ be two scalar-valued kernels and $K (\theta',\theta) =
    k_{\hyperparameterspace}(\theta,\theta')I_{\mcH_{k_{\mcX}}}$.  Then $H_K$
    is isometric to $\mcH_{k_{\mcX}} \otimes \mcH_{k_{\hyperparameterspace}}$
    by means of the isometry $W: f \otimes g \in \mcH_{k_{\mcX}} \otimes
    \mcH_{k_{\hyperparameterspace}} \mapsto ( \hyperparameter \mapsto
    g(\hyperparameter)f) \in \mcH_K$.
\end{lemma}
\begin{remark}
    Given $k_{\mcX}:\inputspace \times \inputspace \rightarrow \reals$,
    $k_{\hyperparameterspace}:\hyperparameterspace \times \hyperparameterspace
    \rightarrow \reals$ two scalar-valued kernels, we define $K:(x,z) \in \mcX
    \times \mcX \mapsto k_{\mcX}(x,z) I_{\mcH_{k_{\hyperparameterspace}}} \in
    \mathcal{L}(\mcH_{k_{\hyperparameterspace}}  )$, $K':
    (\hyperparameter,\hyperparameter')\in \hyperparameterspace \times
    \hyperparameterspace \mapsto
    k_{\hyperparameterspace}(\hyperparameter,\hyperparameter')
    I_{\mcH_{k_{\mcX}}} \in \mathcal{L}(\mcH_{k_{\mcX}})$.
    \cref{lemma:isometry} allows us to say that $\mcH_{K}$ and $\mcH_{K'}$ are
    isometric by means of the isometry \begin{align}
    \label{equation:v_definition} W: h \in \mcH_{K'} \mapsto (x \mapsto (\theta
    \mapsto h(\theta)(x))) \in \mcH_K.  \end{align}
\end{remark}

\begin{lemma} \label{lemma:decompo_ortho}
    Let $k_{\mcX}:\inputspace \times \inputspace \rightarrow \reals$,
    $k_{\hyperparameterspace}:\hyperparameterspace \times \hyperparameterspace
    \rightarrow \reals$  be two scalar-valued kernels and $K : (\theta,\theta')
    \mapsto k_{\hyperparameterspace}(\theta,\theta')I_{\mcH_{k_{\mcX}}}$.  For
    $\hyperparameter \in \hyperparameterspace$, define $K_{\hyperparameter}: f
    \in \mcH_{k_{\mcX}} \mapsto \left ( \hyperparameter' \mapsto
    K(\hyperparameter',\hyperparameter)f \right ) \in \mcH_K$.  It is easy to
    see that $K_{\hyperparameter}^*$ is the evaluation operator
    $K_{\hyperparameter}^*: h \in \mcH_K \mapsto h(\hyperparameter) \in
    \mcH_{k_{\mcX}}$.  Then $\forall m \in \mathbb{N}^*, \forall
    (\theta_j)_{j=1}^m \in \Theta^m$,
    \begin{dmath} \label{equation:decompo_ortho}
        \left ( +_{j=1}^m \Image(K_{\theta_j}) \right) \oplus \left
        (\cap_{j=1}^m \Nullspace(K_{\theta_j} ^*) \right ) = \mcH_K
    \end{dmath}
\end{lemma}
\begin{proof}
    The statement boils down to proving that $\mcV :=\left ( +_{j=1}^m
    \Image(K_{\theta_j}) \right)$ is closed in $\mcH_K$, since it is
    straightforward that $\mcV ^{\perp} = \left (\cap_{j=1}^m \Nullspace
    \left(K_{\theta_j} ^*\right ) \right)$.  Let  $\left(e_j \right)_{j=1}^k$
    be an orthonormal basis of
    $\sspan{\Set{(k_{\hyperparameterspace}(\cdot,\hyperparameter_j))_{j=1}^m}}
    \subset \mcH_{k_{\hyperparameterspace}}$. Such basis can be obtained by
    applying the Gram-Schmidt orthonormalization method to
    $(k_{\hyperparameterspace}(\cdot,\hyperparameter_j))_{j=1}^m$. Then, $V =
    \sspan{ \Set{ e_j \cdot f , 1 \leq j \leq k, f \in \mcH_{k_{\mcX}}}}$.
    Notice also that $1 \leq j,l \leq k, \forall f,g \in \mcH_{k_{\mcX}}$,
    \begin{dmath} \label{equation:scalar_tensor} \langle e_j \cdot f, e_l \cdot
    g \rangle_{\mcH_K} = \langle e_j , e_l
    \rangle_{\mcH_{k_{\hyperparameterspace}}} \cdot \langle f,g
    \rangle_{\mcH_{k_{\mcX}}} \end{dmath} Let $(h_n)_{n \in \mathbb{N}^*}$ be a
    sequence in $\mcV$ converging to some $h \in \mcH_K$. By definition, one
    can find sequences $(f_{1,n})_{n \in \mathbb{N}^*},\ldots,(f_{k,n})_{n \in
    \mathbb{N}^*} \in \mcH_{k_{\mcX}}$ such that $\forall n \in \mathbb{N}^*$,
    $h_n = \sum_{j=1}^k e_j \cdot f_{n,j}$.  Let $p,q \in \mathbb{N}^*$. It
    holds that, using the orthonormal property of $\left(e_j \right)_{j=1}^k$
    and \cref{equation:scalar_tensor}, $\norm{h_p - h_q}_{\mcH_K}^2 = \norm{
    \sum_{j=1}^k e_j (f_{j,p} - f_{j,q})}_{\mcH_K}^2 = \sum_{j=1}^k \norm{
    f_{j,p} - f_{j,q}}_{\mcH_{k_{\mcX}}}^2$. $(h_n)_{n \in \mathbb{N}^*}$ being
    convergent, it is a Cauchy sequence, thus so are the sequences
    $(f_{j,n})_{n \in \mathbb{N}^*}$. But $\mcH_{k_{\mcX}}$ is a complete
    space, so these sequences are convergent in $\mcH_{k_{\mcX}}$, and by
    denoting $f_j = \lim_{n \to \infty} f_{j,n}$, one gets $h = \sum_{j=1}^k
    e_k \cdot f_j$.  Therefore $h \in \mcV$, $\mcV$ is closed and the
    orthogonal decomposition \cref{equation:decompo_ortho} holds.
\end{proof}
\begin{lemma} \label{lemma:coercivity}
    Let $k_{\mcX},k_{\hyperparameterspace}$ be two scalar kernels and $K :
    (\theta,\theta') \mapsto
    k_{\hyperparameterspace}(\theta,\theta')I_{\mcH_{k_{\mcX}}}$.  Let also $m
    \in \mathbb{N}^*$ and $(\theta_j)_{j=1}^m \in \Theta^m$, and $\mcV = \left
    ( +_{j=1}^m \Image(K_{\theta_j}) \right)$. Then  $I: \mcV \rightarrow
    \reals$ defined as $I (h) =  \sum_{j=1}^m
    \norm{h(\hyperparameter_j)}_{\mcH_{k_{\mcX}}}^2$ is coercive.
\end{lemma}
\begin{proof}
    Notice first that if there exists $\theta_j$ such that
    $k_{\hyperparameterspace}(\theta_j,\theta_j) = 0$, then
    $\Image(K_{\theta_j})=0 $, so without loss of generality, we assume that
    $k_{\hyperparameterspace}(\theta_j,\theta_j) > 0 $ ($1 \leq j \leq m$).
    Notice that $I$ is the quadratic form associated to the $L:\mcH_K
    \rightarrow \mcH_K$ linear mapping $ L (h) = \sum_{j=1}^m
    K_{\hyperparameter_j}K_{\hyperparameter_j}^*$.  Indeed, $\forall h \in
    \mcV$, $I(h) = \sum_{j=1}^m \langle
    K_{\hyperparameter_j}^*h,K_{\hyperparameter_j}^*h \rangle_{\mcH_{k_{\mcX}}}
    = \sum_{j=1}^m \langle h, K_{\hyperparameter_j}K_{\hyperparameter_j}^* h
    \rangle_{\mcH_K} = \langle h , L h \rangle_{\mcH_K}$.  Moreover, $\forall 1
    \leq j \leq m$, $K_{\hyperparameter_j}K_{\hyperparameter_j}^*$ has the same
    eigenvalues as $K_{\hyperparameter_j}^*K_{\hyperparameter_j}$, and $\forall
    f \in \mcH_{k_{\mcX}}$, $K_{\hyperparameter_j}^*K_{\hyperparameter_j} f =
    k_{\hyperparameterspace}(\hyperparameter_j,\hyperparameter_j)f$, so that
    the only possible eigenvalue is
    $k_{\hyperparameterspace}(\hyperparameter_j,\hyperparameter_j)$.  Let $h
    \in \mcV$, $h \neq 0$. Because of the \cref{equation:decompo_ortho}, $h$
    cannot be simultaneously in all $\Nullspace(K_{\hyperparameter_j}^*)$, and
    there exists $i_0$ such that $I(h) \geq
    k_{\hyperparameterspace}(\hyperparameter_{i_0},\hyperparameter_{i_0})
    \norm{h}_{\mcH_K}^2$.  Let $\gamma = \underset{1 \leq j \leq m}{\min}
    k_{\hyperparameterspace}(\hyperparameter_{j},\hyperparameter_{j})$.  By
    assumption $\gamma >0$, and it holds that $\forall h \in \mcV$, $I(h) \geq
    \gamma \norm{h}_{\mcH_K}^2$, which proves the coercivity of $I$.
\end{proof}
5
\begin{proof} [Proof of \cref{theorem:representer_ocsvm}]
    Let $K: (x,z) \in \mcX \times \mcX \mapsto k_{\mcX}(x,z)
    I_{\mcH_{k_{\hyperparameterspace}}} \in
    \mathcal{L}(\mcH_{k_{\hyperparameterspace}}  )$, $K':
    (\hyperparameter,\hyperparameter') \in \hyperparameterspace \times
    \hyperparameterspace \mapsto
    k_{\hyperparameterspace}(\hyperparameter,\hyperparameter')
    I_{\mcH_{k_{\mcX}}} \in \mathcal{L}(\mcH_{k_{\mcX}})$, and define
    \begin{dmath*}
        J \colon
        \begin{cases}
            \mcH_K \times \mcH_{k_{b}} & \to \mathbb{R}  \\
            (h,t)  & \mapsto   \frac{1}{n} \displaystyle\sum_{i,j=1}^{n,m}
            \frac{w_j}{\hyperparameter_j} \abs{t(\hyperparameter_j) -
            h(x_i)(\hyperparameter_j)}_{+} + \displaystyle\sum_{j=1}^m w_j
            \left (\norm{h(\cdot)(\hyperparameter_j)}_{\mcH_{k_{\mcX}}}^2 -
            t(\hyperparameter_j) \right ) + \frac{\lambda}{2}
            \norm{t}_{\mcH_{k_{b}}}^2.
        \end{cases}
    \end{dmath*}
  Let $\mcV = W \left ( +_{j=1}^m \Image(K_{\theta_j}') \right) $ where
  $W \colon \mcH_{K'} \to \mcH_K$ is defined in \cref{equation:v_definition}. Since $W$ is an isometry,
  thanks to \cref{equation:decompo_ortho}, it holds that
  $\mcV \oplus \mcV^{\perp} = \mcH_K$.
    Let $(h,t) \in \mcH_K \times \mcH_{k_{b}}$, there exists unique $ h_{\mcV^{\perp}} \in \mcV^{\perp}$,
$ h_{\mcV} \in \mcV$ such that $h = h_{\mcV} + h_{\mcV^{\perp}}$. Notice that
  $J(h,t) = J(h_{\mcV} + h_{\mcV^{\perp}},t) = J(h_{\mcV},t)$
since $\forall 1 \leq j \leq m, \forall x \in \mcX$,
$h_{\mcV^{\perp}}(x)(\hyperparameter_j) = W^{-1}h_{\mcV^{\perp}}(\hyperparameter_j)(x) = 0$.
Moreover, J is bounded by below so that its infinimum is well-defined, and
  $\underset{(h,t) \in \mcH_K \times \mcH_{k_{b}}}{\inf} J(h,t) =
  \underset{(h,t) \in \mcV \times \mcH_{k_{b}}}{\inf} J(h,t)$.
Finally, notice  that $J$ is coercive on $\mcV \times \mcH_{k_{b}}$ endowed
with the sum of the norm (which makes it a Hilbert space): if
$(h_n,t_n)_{n \in \mathbb{N}^*} \in \mcV \times \mcH_{k_{b}}$ is such that
$\norm{h_n}_{\mcH_K} + \norm{t_n}_{\mcH_{k_{b}}} \underset{n \to \infty}{\to} + \infty$,
then either one has to diverge :
\begin{itemize}
  \item If $\norm{t_n}_{\mcH_{k_{b}}} \underset{n \to \infty}{\to} + \infty$,
  since
  $t_n(\theta_j) = \langle t_n, k_{b}(\cdot,\theta_j) \rangle_{\mcH_{k_{b}}}
    \leq k_{b}(\theta_j,\theta_j) \norm{t_n}_{\mcH_{k_b}} \leq
    \kappa_{b} \norm{t_n}_{\mcH_{k_b}}$ $(\forall 1 \leq j \leq m)$,
  then $J(h_n,t_n) \geq
    \frac{\lambda}{2} \norm{t_n}_{\mcH_{k_{b}}}^2 - \sum_{j=1}^m w_j t(\theta_j)
    \underset{n \to \infty}{\to} + \infty$.
  \item If $\norm{h_n}_{\mcH_{K}} \underset{n \to \infty}{\to} + \infty$,
  according to \cref{lemma:coercivity}, $J(h_n,t_n)\underset{n \to \infty}{\to} + \infty $ as long
  as all $w_j$ are strictly positive.
\end{itemize}

Thus $J$ is coercive, so that \citep[Proposition 11.15]{bauschke2011convex} allows to conclude that
$J$ has a minimizer $(h^*,t^*)$ on $\mcV \times \mcH_{k_{b}}$.
Then, in the same fashion as \cref{decompo_theorem}, define
$\mcU_1 = \sspan{\Set{
(K(\cdot,x_i)k_{\Theta}(\cdot,\theta_j))_{i,j=1}^{n,m} }}
\subset \mcV$ and
$\mcU_2 = \sspan{\Set{
(k_{b}(\cdot,\theta_j))_{j=1}^{m} }}
\subset \mcH_{k_{b}}$,
and use the reproducing property to show that $(h^*,t^*) \in \mcU_1 \times \mcU_2$,
so that there
there exist $\left(\alpha_{ij}\right)_{i,j = 1}^{n,m}$ and
$\left ( \beta_{j} \right )_{j=1}^m$ such that $\forall x,\hyperparameter \in \mcX
    \times \hyperparameterspace$,
     $h^*(x)(\hyperparameter) = \sum_{i,j=1}^{n,m} \alpha_{ij} k_{\mcX}(x,x_i)
      k_{\hyperparameter}(\hyperparameter,\hyperparameter_j)$,
      $t^*(\hyperparameter)  = \sum_{j=1}^{m} \beta_{j} k_{b}(\hyperparameter,\hyperparameter_j)$.
\end{proof}
%%%%%%%%%%%%%%%%%%%%%%%%%%%%%%%%%%%%%%%%%%%%%%%%%%%%%%%%%%%%%%%%%%%%%%%%%%%%%%%%
\section{Generalization error in the context of stability}
\label{appendix:stability}
The analysis of the generalization error will be performed using  the notion of
uniform stability introduced in \citep{bousquet2002stability}. For a derivation
of generalization bounds in \ac{vv-RKHS}, we refer to
\citep{kadri2015operator}.  In their framework, the goal is to minimize a risk
which can be expressed as
\begin{dmath}
    R_{\trainingset,\lambda}(h) = \frac{1}{n} \sum_{i=1}^n \ell(y_i,h,x_i) +
    \lambda \norm{h}_{\mcH_K}^2,
\end{dmath}
where $\trainingset = ((x_1,y_1),\ldots,(x_n,y_n))$ are \ac{iid} inputs and
$\lambda > 0$.  We recover their setting by using losses defined as
\begin{dmath*}
  \ell \colon
  \begin{cases}
      \reals \times \mcH_K \times \mcX &\to ~ \mathbb{R}      \\
      (y,h,x)  & \mapsto  \widetilde{V}(y,f(x)),
  \end{cases}
\end{dmath*}
where $\widetilde{V}$ is a loss associated to some local cost defined in
\cref{equation:integrated_cost}. Then, they study the stability of the
algorithm which, given a dataset $\trainingset$, returns

\begin{dmath} \label{equation:algo}
    \minimizer{h}_{\trainingset} = \argmin_{h \in \mcH_K}
    R_{\trainingset,\lambda}(h).
\end{dmath}

There is a slight difference between their setting and ours, since they use
losses defined for some $y$ in the output space of the \ac{vv-RKHS}, but this
difference has no impact on the validity of the proofs in our case. The use of
their theorem requires some assumption that are listed below. We recall the
shape of the \ac{OVK} we use : $K: (x,z) \in \mcX \times \mcX \mapsto
k_{\mcX}(x,z) I_{\mcH_{k_{\hyperparameterspace}}} \in
\mathcal{L}(\mcH_{k_{\hyperparameterspace}})$, where $k_{\mcX}$ and
$k_{\hyperparameterspace}$ are both bounded scalar-valued kernels, in other
words there exist $(\kappa_{\mcX},\kappa_{\Theta}) \in \reals^2$ such that
$\underset{x \in \mcX}{\sup}~ k_{\mcX}(x,x) < \kappa_{\mcX}^2$ and
$\underset{\theta \in \Theta}{\sup}~ k_{\Theta}(\theta,\theta) <
\kappa_{\Theta}^2$.

\begin{assumption} \label{assumption:1}
    $\exists \kappa > 0$ such that $\forall x \in \mcX$,
    $\norm{K(x,x)}_{\mathcal{L}(\mcH_{k_{\hyperparameterspace}})} \leq
    \kappa^2$.
\end{assumption}
\begin{assumption} \label{assumption:2}
    $\forall h_1,h_2 \in \mcH_{k_{\hyperparameterspace}}$, the function
    $(x_1,x_2) \in \mcX \times \mcX \mapsto \langle K(x_1,x_2) h_1,h_2
    \rangle_{\mcH_{k_{\hyperparameterspace}}} \in \reals$ \condition{is
    measurable.}
\end{assumption}
\begin{remark}
    Assumptions \ref{assumption:1}, \ref{assumption:2} are satisfied for our
    choice of kernel.
\end{remark}
\begin{assumption} \label{assumption:3}
    The application $(y,h,x) \mapsto \ell(y,h,x)$ is $\sigma$-admissible,
    \ac{ie} convex with respect to $f$ and Lipschitz continuous with respect to
    $f(x)$, with $\sigma$ as its Lipschitz constant.
\end{assumption}
\begin{assumption} \label{assumption:4}
    $\exists \xi \geq 0$ such that $\forall (x,y) \in \mcX \times \mcY$ and
    $\forall \trainingset$  training set,
    $ \ell(y,\minimizer{h}_{\trainingset},x) \leq \xi$.
\end{assumption}
\begin{definition}
Let $\trainingset = \left((x_i,y_i)\right)_{i=1}^n$ be the training data.
We call $\trainingset^i$ the training data
$\trainingset^i = ((x_1,y_1),\ldots,(x_{i-1},y_{i-1}),(x_{i+1},y_{i+1}),
\ldots,(x_n,y_n))$, $ 1 \leq i \leq n$.
\end{definition}

\begin{definition}A learning algorithm mapping a dataset
    $\trainingset$ to a function $\minimizer{h}_{\trainingset}$
    is said to be $\beta$-uniformly stable with
    respect to the loss function $\ell$ if $\forall n \geq 1$,
    $\forall 1 \leq i \leq n$, $\forall \trainingset \text{ training set}$,
    $||\ell(\cdot,\minimizer{h}_{\trainingset},\cdot) -
        \ell(\cdot, \minimizer{h}_{\trainingset^{ i}},\cdot)||_{\infty}
        \leq \beta$.
\end{definition}
\begin{proposition} \label{proposition:bousquet_generalization}
    \citep{bousquet2002stability} Let $\trainingset \mapsto
        \minimizer{h}_{\trainingset}$ be a learning algorithm
    with uniform stability $\beta$ with respect to a loss $\ell$ satisfying
    \cref{assumption:4}. Then $\forall n \geq 1$, $\forall \delta \in (0,1)$,
    with probability at least $1-\delta$ on the drawing of the samples, it
    holds that
    \begin{dmath*}
        \risk(\minimizer{h}_{\trainingset}) \leq
        \empiricalrisk(\minimizer{h}_{\trainingset}) + 2 \beta +
        (4 \beta + \xi) \sqrt{\frac{\log{(1/\delta)}}{n}}.
    \end{dmath*}
\end{proposition}

\begin{proposition} \citep{kadri2015operator} \label{proposition:kadri}
Under assumptions
\ref{assumption:1}, \ref{assumption:2}, \ref{assumption:3}, a learning algorithm
that maps a training set $\trainingset$ to the function $\minimizer{h}_{\trainingset}$ defined in
\cref{equation:algo} is $\beta$-stable with $\beta =
\frac{\sigma^2 \kappa^2}{2 \lambda n }$.
\end{proposition}

\subsection{Quantile Regression}
We recall that in this setting, $\hcost(\hyperparameter,y,h(x)(\hyperparameter)) =
\max{(\hyperparameter(y-h(x)(\hyperparameter)),(1-\hyperparameter)(y-h(x)(\hyperparameter)))}$
and the loss is
\begin{dmath}
  \ell \colon
  \begin{cases}
      \reals \times \mcH_K \times \mcX &\to ~ \mathbb{R}      \\
      (y,h,x)  & \mapsto \frac{1}{m} \sum_{j=1}^m \max{(\theta_j(y-h(x)(\theta_j)),(\theta_j-1)(y-h(x)(\theta_j)))}.
  \end{cases}
\end{dmath}
Moreover, we will assume that $|Y|$ is bounded by $B \in \mathbb{R}$ as a \ac{rv}. We will
therefore verify the hypothesis for $y \in [-B,B]$ and not $y \in \reals$.
\begin{lemma} \label{lemma:admissibility_qr}
  In the case of the \ac{QR}, the loss $\ell$ is $\sigma$-admissible
  with $\sigma = 2 \kappa_{\hyperparameterspace}$.
\end{lemma}
\begin{proof}
  Let $h_1,h_2 \in \mcH_K$ and
  $\hyperparameter \in [0,1]$. $\forall x,y \in \mcX \times \reals$, it holds that
  \begin{dmath*}[compact]
  \hcost(\hyperparameter,y,h_1(x)(\hyperparameter)) - \hcost(\hyperparameter,y , h_2(x)(\hyperparameter)) =
  (\hyperparameter -t)(h_2(x)(\hyperparameter) - h_1(x)(\hyperparameter)) + (t-s)(y-h_1(x)(\hyperparameter)),
  \end{dmath*}
  where $s = \mathbf{1}_{y \leq h_1(x)(\hyperparameter)}$ and
  $t = \mathbf{1}_{y \leq h_2(x)(\hyperparameter)}$. We consider all possible cases for
  $t$ and $s$ :
  \begin{compactitem}
    \item $t = s = 0$ : $|(t-s)(y-h_1(x)(\hyperparameter)) |
    \leq |h_2(x)(\hyperparameter) - h_1(x)(\hyperparameter)| $
    \item $t = s = 1$ : $|(t-s)(y-h_1(x)(\hyperparameter)) |
    \leq |h_2(x)(\hyperparameter) - h_1(x)(\hyperparameter)| $
    \item $s=1$,$t=0$ : $|(t-s)(y-h_1(x)(\hyperparameter)) | = |h_1(x)(\hyperparameter) - y| \leq
     |h_1(x)(\hyperparameter) - h_2(x)(\hyperparameter)| $
    \item $s=0$,$t=1$ : $|(t-s)(y-h_1(x)(\hyperparameter)) | = |y - h_1(x)(\hyperparameter)| \leq
    |h_1(x)(\hyperparameter) - h_2(x)(\hyperparameter)|$ because of the conditions on $t,s$.
  \end{compactitem}
  Thus $|\hcost(\hyperparameter,y,h_1(x)(\hyperparameter)) - \hcost(\hyperparameter,y , h_2(x)(\hyperparameter))| \leq
  (\hyperparameter + 1) | h_1(x)(\hyperparameter) - h_2(x)(\hyperparameter) | \leq (\hyperparameter + 1) \kappa_{\hyperparameterspace}
  || h_1(x) - h_2(x) ||_{\mcH_{k_{\hyperparameterspace}}}$.
  By summing this expression over the $(\theta_j)_{j=1}^m$, we get that
  \begin{dmath*}[compact]
  |\ell(x,h_1,y) - \ell(x,h_2,y)| \leq \frac{1}{m} \sum_{j=1}^m (\hyperparameter_j+1) \kappa_{\hyperparameterspace}
  || h_1(x) - h_2(x) ||_{\mcH_{k_{\hyperparameterspace}}} \leq
  2 \kappa_{\hyperparameterspace} ||h_1(x) - h_2(x) ||_{\mcH_{k_{\hyperparameterspace}}}
  \end{dmath*}
  and $\ell$ is $\sigma$-admissible with $\sigma = 2 \kappa_{\hyperparameterspace}$.
\end{proof}

\begin{lemma} \label{lemma:majorant_h} Let $\trainingset=((x_1,y_1),\ldots,(x_n,y_n))$ be a training set and
$\lambda > 0$. Then $\forall x , \hyperparameter \in \mcX \times (0,1)$, it holds that
$|\minimizer{h}_{\trainingset}(x)(\hyperparameter)| \leq \kappa_{\mcX} \kappa_{\hyperparameterspace} \sqrt{\frac{B}{\lambda}}$.
\end{lemma}
\begin{proof} Since $\minimizer{h}_{\trainingset}$ is the output of our algorithm and $0 \in \mcH_K$,
it holds that
\begin{dmath*}[compact]
\lambda ||\minimizer{h}_{\trainingset}||^2 \leq \frac{1}{nm} \sum_{i=1}^n \sum_{j=1}^m  \hcost(\hyperparameter_j,y_i,0)
\leq \frac{1}{nm} \sum_{i=1}^n \sum_{j=1}^m \max{(\hyperparameter_j,1-\hyperparameter_j)} |y_i|
\leq B.
\end{dmath*}
Thus $||\minimizer{h}_{\trainingset}|| \leq \sqrt{\frac{B}{\lambda}}$. Moreover,
$\forall x , \hyperparameter \in \mcX \times (0,1)$,
$|\minimizer{h}_{\trainingset}(x)(\hyperparameter)| =
|\langle \minimizer{h}_{\trainingset}(x),k_{\hyperparameterspace}(\hyperparameter,\cdot) \rangle_{\mcH_{k_{\hyperparameterspace}}}|
\leq ||\minimizer{h}_{\trainingset}(x)||_{\mcH_{k_{\hyperparameterspace}}} \kappa_{\hyperparameterspace}
\leq ||\minimizer{h}_{\trainingset}||_{\mcH_{k_{\hyperparameterspace}}} \kappa_{\mcX} \kappa_{\hyperparameterspace}$
which concludes the proof.
\end{proof}

\begin{lemma} \label{lemma:xi_qr} Assumption \ref{assumption:4} is satisfied for
  $\xi = 2\left(B + \kappa_{\mcX} \kappa_{\hyperparameterspace} \sqrt{\frac{B}{\lambda}}\right)$.
\end{lemma}

\begin{proof}Let $\trainingset = ((x_1,y_1),\ldots,(x_n,y_n))$ be a training set and
$\minimizer{h}_{\trainingset}$ be the output of our algorithm.
  $\forall (x,y) \in \mcX \times [-B,B]$, it holds that
  \begin{align*}
    \ell(y,\minimizer{h}_{\trainingset},x) &=
    \frac{1}{m} \sum_{j=1}^m \max{(\theta_j(y-\minimizer{h}_{\trainingset}(x)(\theta_j)),(\theta_j-1)(y-\minimizer{h}_{\trainingset}(x)(\theta_j)))}
    \leq \frac{2}{m} \sum_{j=1}^m |y-\minimizer{h}_{\trainingset}(x)(\theta_j)| \\
    &\leq \frac{2}{m} \sum_{j=1}^m |y| + |\minimizer{h}_{\trainingset}(x)(\theta_j)|
    \leq 2\left (B + \kappa_{\mcX} \kappa_{\hyperparameterspace} \sqrt{\frac{B}{\lambda}}\right).
  \end{align*}
\end{proof}

\begin{corollary} \label{corollary:beta_stab_qr}
  The \ac{QR} learning algorithm defined in \cref{equation:h-objective-empir2}
  is such that
  $\forall n \geq 1$, $\forall \delta \in (0,1)$,
  with probability at least $1-\delta$ on the drawing of the samples, it
  holds that
  \begin{dmath} \label{equation:beta_stab_qr}
    \widetilde{\risk}(\minimizer{h}_{\trainingset}) \leq
    \sampledempiricalrisk(\minimizer{h}_{\trainingset})
    + \frac{4 \kappa_{\mcX}^2 \kappa_{\hyperparameterspace}^2}{ \lambda n} +
    \left[\frac{8 \kappa_{\mcX}^2 \kappa_{\hyperparameterspace}^2}{ \lambda n} +
    2\left(B + \kappa_{\mcX} \kappa_{\hyperparameterspace} \sqrt{\frac{B}{\lambda}} \right)\right]
    \sqrt{\frac{\log{(1/\delta)}}{n}}.
  \end{dmath}
  \begin{proof} This is a direct consequence of \cref{proposition:kadri},
    \cref{proposition:bousquet_generalization}, \cref{lemma:admissibility_qr} and
    \cref{lemma:xi_qr}.
  \end{proof}
\end{corollary}

\begin{definition}[Hardy-Krause variation] Let $\Pi$ be the set of subdivisions of the interval $\Theta = [0,1]$.
  A subdivision will be denoted $\sigma = (\theta_1,\theta_2,\ldots,\theta_p)$
  and $f \colon \Theta \to \reals$ be a function.
  We call Hardy-Krause variation of the function $f$ the quantity
    $\underset{\sigma \in \Pi}{\sup} ~ \sum_{i=1}^{p-1} |f(\theta_{i+1}) - f(\theta_i)|$.
\end{definition}
\begin{remark} \label{remark:continuity_mesh}
  If $f$ is continuous, $V(f)$ is also the limit as the mesh of $\sigma$ goes to zero
  of the above quantity.
\end{remark}

In the following, let $f \colon \theta \mapsto \expectation_{X,Y}[
  \hcost(\hyperparameter,Y, \minimizer{h}_{\trainingset}(X)(\hyperparameter))]$.
  This function is of prime importance for our analysis, since in the
  Quasi Monte-Carlo setting, the bound of \cref{proposition:generalization_supervised}
  makes sense only if the function $f$ has finite Hardy-Krause variation, which is
  the focus of the following lemma.

\begin{lemma} \label{lemma:finite_hk} Assume the boundeness of both scalar kernels
  $k_{\mcX}$and $k_{\Theta}$.
  Assume moreover that $k_{\Theta}$ is $\mathcal{C}^1$ and that its partial
  derivatives are uniformly bounded by some constant $C$. Then
  \begin{align}
    V(f) \leq B + \kappa_{\mcX} \kappa_{\hyperparameterspace} \sqrt{\frac{B}{\lambda}} + 2\kappa_{\mcX} \sqrt{\frac{2BC}{\lambda}}.
  \end{align}

\end{lemma}
\begin{proof}
  It holds that
  \begin{dmath*}
    \sup_{\sigma \in \Pi} \sum_{i=1}^{p-1} \abs{f(\theta_{i+1}) -
    f(\theta_i)}
    = \sup_{\sigma \in \Pi} \sum_{i=1}^{p-1} \abs{\int \hcost(\theta_{i+1},y,
    \minimizer{h}_{\trainingset}(x)(\theta_{i+1}))\mathrm{d}\probability_{X,Y}
    - \int \hcost(\theta_{i},y, \minimizer{h}_{\trainingset}(x)(\theta_{i}))
    \mathrm{d}\probability_{X,Y}}
    = \sup_{\sigma \in \Pi} \sum_{i=1}^{p-1} \abs{\int \hcost(\theta_{i+1},y,
    \minimizer{h}_{\trainingset}(x)(\theta_{i+1})) - \hcost(\theta_{i},y,
    \minimizer{h}_{\trainingset}(x)(\theta_{i})) \mathrm{d}\probability_{X,Y}}
    \leq \underset{\sigma \in \Pi}{\sup} ~ \sum_{i=1}^{p-1} \int
    \abs{\hcost(\theta_{i+1},y, \minimizer{h}_{\trainingset}(x)(\theta_{i+1})) -
    \hcost(\theta_{i},y, \minimizer{h}_{\trainingset}(x)(\theta_{i}))
    }\mathrm{d}\probability_{X,Y}
    \leq \sup_{\sigma \in \Pi} \int \sum_{i=1}^{p-1} \abs{\hcost(\theta_{i+1},y,
    \minimizer{h}_{\trainingset}(x)(\theta_{i+1})) - \hcost(\theta_{i},y,
    \minimizer{h}_{\trainingset}(x)(\theta_{i})) }\mathrm{d}\probability_{X,Y}.
  \end{dmath*}
  The supremum of the integral is lesser than the integral of the supremum, as
  such
  \begin{dmath} \label{equation:darboux_hk}
    V(f) \leq \int V(f_{x,y}) \mathrm{d} \probability_{X,Y},
  \end{dmath}
  where $f_{x,y} \colon \theta \mapsto \hcost(\theta,y,
  \minimizer{h}_{\trainingset}(x)(\theta))$ is the counterpart of the function
  $f$ at point $(x,y)$. To bound this quantity, let us first bound locally $
  V(f_{x,y})$. To that extent, we fix some $(x,y)$ in the following.  Since
  $f_{x,y}$ is continuous (because $k_{\Theta}$ is $\mathcal{C}^1$), then using
  \citet[Theorem 24.6]{choquet1969cours}, it holds that
  \begin{dmath*}
    V(f_{x,y}) = \lim_{\abs{\sigma}\to 0} \sum_{i=1}^{p-1}
    \abs{f_{x,y}(\theta_{i+1}) - f_{x,y}(\theta_{i})}.
  \end{dmath*}
  Moreover since $k\in\mathcal{C}^1$ and $\partial k_\theta = (\partial_1
  k)(\cdot, \theta)$ has a finite number of zeros for all
  $\theta\in\mathcal{\hyperparameterspace}$, one can assume that in the
  subdivision considered afterhand all the zeros (in $\theta$) of the residuals
  $y - \minimizer{h}_{\trainingset}(x)(\theta) $ are present, so that $y
  -\minimizer{h}_{\trainingset}(x)(\theta_{i+1})$ and $y -
  \minimizer{h}_{\trainingset}(x)(\theta_{i})$ are always of the same sign.
  Indeed, if not, create a new, finer subdivision with this property and work
  with this one. Let us begin the proper calculation: let $\sigma =
  (\theta_1,\theta_2,\ldots,\theta_p)$ be a subdivision of $\Theta$, it holds
  that $\forall i \in \Set{1,\ldots,p-1}$:
  \begin{dmath*}
     \abs{f_{x,y}(\theta_{i+1}) - f_{x,y}(\theta_i)}
    =
    |\max{(\theta_{i+1}(y-\minimizer{h}_{\trainingset}(x)(\theta_{i+1})),
    (1-\theta_{i+1})(y-\minimizer{h}_{\trainingset}(x)(\theta_{i+1})))} \quad -
    \max{(\theta_{i}(y-\minimizer{h}_{\trainingset}(x)(\theta_{i})),
    (1-\theta_{i+1})(y-\minimizer{h}_{\trainingset}(x)(\theta_{i})))}|.
  \end{dmath*}
  We now study the two possible outcomes for the residuals:
  \begin{itemize}
    \item If $y-h(x)(\theta_{i+1}) \geq 0$ and $y-h(x)(\theta_{i}) \geq 0$ then
    \begin{dmath*}
      \abs{f_{x,y}(\theta_{i+1}) - f_{x,y}(\theta_i)} =
      \abs{\theta_{i+1}(y-\minimizer{h}_{\trainingset}(x)(\theta_{i+1})) -
      \theta_{i}(y-\minimizer{h}_{\trainingset}(x)(\theta_{i}))}
      = \abs{(\theta_{i+1} - \theta_i)y + (\theta_i - \theta_{i+1})\minimizer{h}_{\trainingset}(x)(\theta_{i+1})
      + \theta_i (\minimizer{h}_{\trainingset}(x)(\theta_{i})
      - \minimizer{h}_{\trainingset}(x)(\theta_{i+1}))}
      \leq |(\theta_{i+1} - \theta_i)y| + |(\theta_i - \theta_{i+1})\minimizer{h}_{\trainingset}(x)(\theta_{i+1})|
      + |\theta_i (\minimizer{h}_{\trainingset}(x)(\theta_{i})
      - \minimizer{h}_{\trainingset}(x)(\theta_{i+1}))|
    \end{dmath*}
    From \cref{lemma:majorant_h}, it holds that
    $\minimizer{h}_{\trainingset}(x)(\theta_{i+1}) \leq \kappa_{\mcX}
    \kappa_{\hyperparameterspace} \sqrt{\frac{B}{\lambda}}$.  Moreover,
    \begin{dmath*}
      \abs{\minimizer{h}_{\trainingset}(x)(\theta_{i})
      - \minimizer{h}_{\trainingset}(x)(\theta_{i+1})}
      = \abs{\inner{h(x) , k_{\Theta}(\theta_{i},\cdot) -
      k_{\Theta}(\theta_{i+1},\cdot)}_{\mcH_{k_{\Theta}}}}
      \leq \norm{h(x)}_{\mcH_{k_{\Theta}}} \norm{k_{\Theta}(\theta_{i},\cdot) -
      k_{\Theta}(\theta_{i+1},\cdot)}_{\mcH_{k_{\Theta}}}
      \leq \kappa_{\mcX} \sqrt{\frac{B}{\lambda}} \sqrt{\abs{
      k_{\Theta}(\theta_{i},\theta_i) + k_{\Theta}(\theta_{i+1},\theta_{i+1})
      - 2  k_{\Theta}(\theta_{i+1},\theta_{i}) }}
      \leq  \kappa_{\mcX} \sqrt{\frac{B}{\lambda}} \left (
      \sqrt{\abs{k_{\Theta}(\theta_{i+1},\theta_{i+1}) -
      k_{\Theta}(\theta_{i+1},\theta_{i})}} +
       \sqrt{\abs{k_{\Theta}(\theta_{i},\theta_{i}) -
       k_{\Theta}(\theta_{i+1},\theta_{i}) }} \right ).
    \end{dmath*}
    Since $k_{\Theta}$ is $\mathcal{C}^1$, with partial derivatives uniformly
    bounded by $C$, $\abs{k_{\Theta}(\theta_{i+1},\theta_{i+1}) -
    k_{\Theta}(\theta_{i+1},\theta_{i})} \leq C(\theta_{i+1}-\theta_i)$ and $
    \abs{k_{\Theta}(\theta_{i},\theta_{i}) -
    k_{\Theta}(\theta_{i+1},\theta_{i})} \leq C(\theta_{i+1}-\theta_i)$ so that
    $\abs{\minimizer{h}_{\trainingset}(x)(\theta_{i}) -
    \minimizer{h}_{\trainingset}(x)(\theta_{i+1})} \leq \kappa_{\mcX}
    \sqrt{\frac{2BC}{\lambda}} \sqrt{\theta_{i+1}-\theta_i}$
    and overall
    \begin{dmath*}
      \abs{f_{x,y}(\theta_{i+1}) - f_{x,y}(\theta_i)} \leq \left( B +
      \kappa_{\mcX} \kappa_{\hyperparameterspace} \sqrt{\frac{B}{\lambda}}
      \right) (\theta_{i+1} - \theta_i) + \kappa_{\mcX}
      \sqrt{\frac{2BC}{\lambda}} \sqrt{\theta_{i+1}-\theta_i}.
    \end{dmath*}
    \item If $y-h(x)(\theta_{i+1}) \leq 0$ and $y-h(x)(\theta_{i}) \leq 0$ then
    $\abs{f_{x,y}(\theta_{i+1}) - f_{x,y}(\theta_i)} =
    \abs{(1-\theta_{i+1})(y-\minimizer{h}_{\trainingset}(x)(\theta_{i+1})) -
    (1-\theta_{i})(y-\minimizer{h}_{\trainingset}(x)(\theta_{i}))}
    \leq \abs{\minimizer{h}_{\trainingset}(x)(\theta_{i}) -
    \minimizer{h}_{\trainingset}(x)(\theta_{i+1})} +
    \abs{(\theta_{i+1} - \theta_i)y} + \abs{(\theta_i -
    \theta_{i+1})\minimizer{h}_{\trainingset}(x)(\theta_{i+1})}
    + \abs{\theta_i (\minimizer{h}_{\trainingset}(x)(\theta_{i})
    - \minimizer{h}_{\trainingset}(x)(\theta_{i+1}))}$
  so that with similar arguments one gets
  \begin{dmath} \label{equation:sign_hk}
    \abs{f_{x,y}(\theta_{i+1}) - f_{x,y}(\theta_i)} \leq \left( B +
    \kappa_{\mcX} \kappa_{\hyperparameterspace} \sqrt{\frac{B}{\lambda}}
    \right) (\theta_{i+1} - \theta_i) + 2\kappa_{\mcX}
    \sqrt{\frac{2BC}{\lambda}} \sqrt{\theta_{i+1}-\theta_i}.
  \end{dmath}
  \end{itemize}
  Therefore, regardless of the sign of the residuals $y-h(x)(\theta_{i+1})$ and
  $y-h(x)(\theta_{i})$, one gets \cref{equation:sign_hk}. Since the square root
  function has Hardy-Kraus variation of $1$ on the interval $\Theta = [0,1]$,
  it holds that
  \begin{dmath*}
    \underset{\sigma \in \Pi}{\sup} \sum_{i=1}^{p-1} |f_{x,y}(\theta_{i+1}) -
    f_{x,y}(\theta_i)| \leq B + \kappa_{\mcX} \kappa_{\hyperparameterspace}
    \sqrt{\frac{B}{\lambda}} + 2\kappa_{\mcX} \sqrt{\frac{2BC}{\lambda}}.
  \end{dmath*}
  Combining this with \cref{equation:darboux_hk} finally gives
  \begin{dmath*}
    V(f) \leq  B + \kappa_{\mcX} \kappa_{\hyperparameterspace}
    \sqrt{\frac{B}{\lambda}} + 2\kappa_{\mcX} \sqrt{\frac{2BC}{\lambda}}.
  \end{dmath*}
\end{proof}

\begin{lemma} \label{lemma:hardy_krause} Let $R$ be the risk defined in
\cref{equation:h-objective} for the quantile regression problem. Assume that
the $(\theta)_{j=1}^m$ have been generated via the Sobol sequence and that
$k_{\Theta}$ is $\mathcal{C}^1$ and that its partial derivatives are uniformly
bounded by some constant $C$.
  % and that the kernel $\mcH_{k_{\hyperparameterspace}}$ is
  % $\mathcal{C}^1$ with uniformly bounded partial derivatives over the space, that is
  % $\exists C \in \mathbb{R}$, $\forall (\theta_1,\theta_2) \in \Theta$,
  % $\frac{\partial k_{\Theta}}{\partial x}(\theta_1,\theta_2) \leq C$.
  Then
  % $|R(\minimizer{h}_{\trainingset}) - \widetilde{R}(\minimizer{h}_{\trainingset})|
    % = \mathcal{O}\left(\frac{\log(m)}{m}\right)$.
    \begin{align}
      |R(\minimizer{h}_{\trainingset}) - \widetilde{R}(\minimizer{h}_{\trainingset})| \leq
      \left (B + \kappa_{\mcX} \kappa_{\hyperparameterspace} \sqrt{\frac{B}{\lambda}} + 2\kappa_{\mcX} \sqrt{\frac{2BC}{\lambda}} \right) \frac{\log(m)}{m}
    \end{align}
  \end{lemma}
  \begin{proof}
    Let $f \colon \theta \mapsto
    \expectation_{X,Y}[
    \hcost(\hyperparameter,Y, \minimizer{h}_{\trainingset}(X)(\hyperparameter))]$. It holds that
    $|R(\minimizer{h}_{\trainingset}) - \widetilde{R}(\minimizer{h}_{\trainingset})|
      \leq V(f) \frac{\log(m)}{m}$
    according to classical Quasi-Monte Carlo approximation results, where $V(f)$
    is the Hardy-Krause variation of $f$. \cref{lemma:finite_hk} allows then to conclude.
     \end{proof}
\begin{proof} [Proof of \cref{proposition:generalization_supervised}]
  Combine \cref{lemma:hardy_krause} and \cref{corollary:beta_stab_qr} to
  get an asymptotic behaviour as $n,m \to \infty$.
\end{proof}
\subsection{Cost-Sensitive Classification}
In this setting, the cost is $\hcost(\hyperparameter,y,h(x)(\hyperparameter)) =
\abs{\frac{\theta + 1}{2} - \indicator{\Set{-1}(y)}}\abs{1
- yh_{\theta}(x)}_{+}$
and the loss is
\begin{dmath*}
  \ell \colon
  \begin{cases}
      \reals \times \mcH_K \times \mcX &\to ~ \mathbb{R}      \\
      (y,h,x)  & \mapsto \frac{1}{m} \sum_{j=1}^m \abs{\frac{\theta_j + 1}{2} -
      \indicator{\Set{-1}}(y)}\abs{1
      - yh_{\theta_j}(x)}_{+}.
  \end{cases}
\end{dmath*}
It is easy to verify in the same fashion as for \ac{QR}
that the properties above still hold, but with constants
    $\sigma  =  \kappa_{\hyperparameterspace}$, $\beta  = \frac{\kappa_{\mcX}^2 \kappa_{\hyperparameterspace}^2}{2 \lambda n}$,
    $\xi  = 1 + \frac{\kappa_{\mcX}\kappa_{\hyperparameterspace}}{\sqrt{\lambda}}$.
so that we get analogous properties to \ac{QR}.
\begin{corollary} \label{corollary:beta_stab_csc}
  The \ac{CSC} learning algorithm defined in \cref{equation:h-objective-empir2}
  is such that
  $\forall n \geq 1$, $\forall \delta \in (0,1)$,
  with probability at least $1-\delta$ on the drawing of the samples, it
  holds that
  \begin{dmath*} %\label{equation:beta_stab_csc}
    \widetilde{\risk}(\minimizer{h}_{\trainingset}) \leq
    \sampledempiricalrisk(\minimizer{h}_{\trainingset})
   + \frac{\kappa_{\mcX}^2 \kappa_{\hyperparameterspace}^2}{ \lambda n} +
    \left(\frac{2 \kappa_{\mcX}^2 \kappa_{\hyperparameterspace}^2}{ \lambda n} +
     1 + \frac{\kappa_{\mcX}\kappa_{\hyperparameterspace}}{\sqrt{\lambda}}
      \right) \sqrt{\frac{\log{(1/\delta)}}{n}}.
  \end{dmath*}
\end{corollary}
%%%%%%%%%%%%%%%%%%%%%%%%%%%%%%%%%%%%%%%%%%%%%%%%%%%%%%%%%%%%%%%%%%%%%%%%%%%%%%%%
\section{Experimental remarks}
\label{appendix:experiments}
We present here more details on the experimental protocol used in the main
paper as well as new experiments
\subsection{Alternative hyperparameters sampling}\label{subsection:sampling}
Many quadrature rules such as \ac{MC} and \ac{QMC} methods are well suited for
\acl{ITL}. For instance when $\hyperparameterspace$ is high dimensional,
\ac{MC} is typically prefered over \ac{QMC}, and vice versa. If
$\hyperparameterspace$ is one dimensional and the function to integrate is
smooth enough then a Gauss-Legendre quadrature would be preferable. In
\cref{sec:V} of the main paper we provide a unified notation to handle
\ac{MC}, \ac{QMC} and other quadrature rules. In the case of
\begin{itemize}
    \item \ac{MC}: $w_j = \frac{1}{m}$ and $(\theta_j)_{j=1}^m
    \sim \mu^{\otimes m}$.
    \item \ac{QMC}: $w_j = m^{-1}F^{-1}(\theta_j)$ and $(\theta_j)_{j=1}^m$
    is a sequence with values in $[0, 1]^d$ such as the % uniform properties
    Sobol or Halton sequence, $\mu$ is assumed to be absolutely continuous
    \acs{wrt} the Lebesgue measure, $F$ is the associated cdf.
    \item Quadrature rules: $((\theta_j, w'_j))_{j=1}^m$ is the indexed set of
    locations and weights produced by the quadrature rule, $w_j
    = w'_j f_{\mu}(\theta_j)$, $\mu$ is assumed to be absolutely continuous
    \acs{wrt} the Lebesgue measure, and $f_\mu$ denotes its corresponding
    probability density function.
\end{itemize}
\subsection{Impact of the number of hyperparameters sampled}
\begin{figure}[!htbp]
    \centering
    \includegraphics[width=\textwidth]{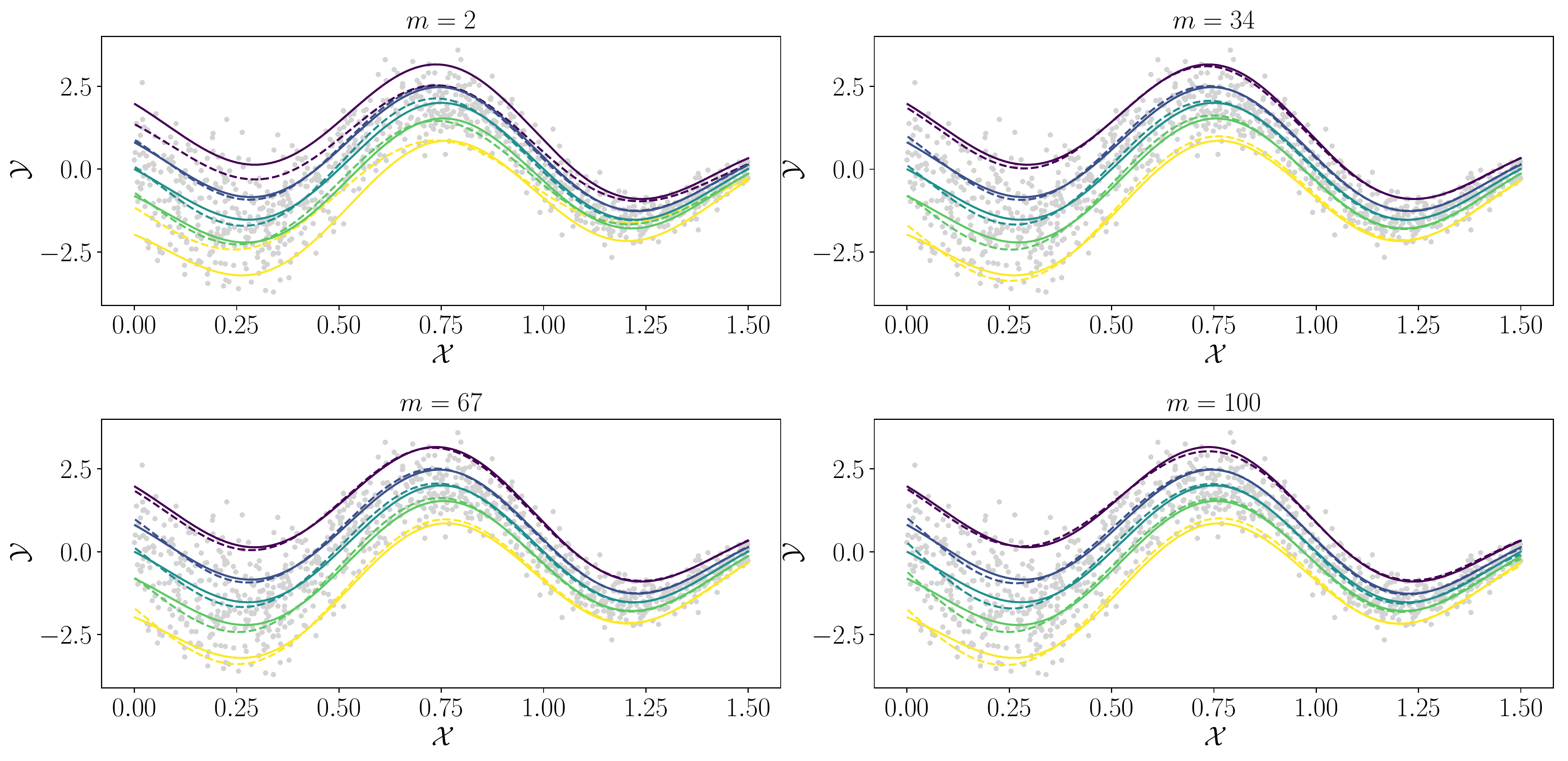}
    \caption{Impact of the number of hyperparameters sampled.
             \label{figure:iqr_m}}
\end{figure}
In the experiment presented on \cref{figure:iqr_m}, on the sine synthetic
benchmark, we draw $n=1000$ training points and study the impact of increasing
$m$ on the quality of the quantiles at $\theta\in\Set{0.05, 0.25, 0.5, 0.75,
0.95}$. We notice that when $m\ge34\approx\sqrt{1000}$ there is little benefit
to draw more $m$ samples are the quantile curves do not change on the
$n_{\text{test}}=2000$ test points.
\subsection{Smoothifying the cost function} \label{subsection:infimal_convo}
The resulting $\kappa$-smoothed ($\kappa\in\reals_+$) absolute value
($\psi_{1}^\kappa$) and positive part ($\psi_{+}^\kappa$) are as follows:
\begin{align*}
  \psi_{1}^\kappa(p) &\defeq \left(\kappa \abs{\cdot} \square \frac{1}{2}
    \abs{\cdot}^2 \right)(p)
     =
     \begin{cases}
        \frac{1}{2\kappa}p^2 & \text{if $\abs{p} \le \kappa$} \\
        \abs{p} - \frac{\kappa}{2} & \text{otherwise},
    \end{cases}\\
     \psi_{+}^\kappa(p) &\defeq \left(\kappa\abs{\cdot}_+ \square\frac{1}{2}
     \abs{\cdot}^2 \right)(p) =
     \begin{cases}
         \frac{1}{2\kappa}\abs{p}_+^2 & \text{if $p \le \kappa$} \\
         p - \frac{\kappa}{2} & \text{otherwise.}
     \end{cases}
\end{align*}
where $\square$ is the classical infimal convolution \citep{bauschke2011convex}.
All the smoothified loss functions used in this paper
have been gathered in \cref{table:integrated_risks}.
\paragraph{Remarks}
\begin{itemize}[labelindent=0cm,leftmargin=*,topsep=0cm,partopsep=0cm,
                parsep=0cm,itemsep=0cm]
    \item Minimizing the $\kappa$-smoothed pinball loss
    \begin{align*}
     \parametrizedcost{\hyperparameter,\kappa}(y, h(x))=\abs{\hyperparameter -
     \indicator{\reals_{-}}(y - h(x))}\psi_1^\kappa(y - h(x)),
    \end{align*}
    yields the quantiles when $\kappa\rightarrow 0$, the expectiles as
    $\kappa\to+\infty$. The intermediate values are known as M-quantiles
    \citep{breckling1988m}.
    \item In practice, the absolute value and positive part can be approximated
    by a smooth function by setting the smoothing parameter $\kappa$ to be a
    small positive value;
    the optimization showed a robust behaviour \acs{wrt} this choice with a
    random coefficient initialization.
\end{itemize}\par
\paragraph{Impact of the huber loss support}
The influence of the $\kappa$ parameter is illustrated in
\cref{figure:kappa_study}.  For this experiment, $10000$ samples have been
generated from the sine wave dataset described in \cref{paragraph:datasets},
and the model have been trained on $100$ quantiles generated from a
Gauss-Legendre Quadrature.  When $\kappa$ is large the expectiles are learnt
(dashed lines) while when $\kappa$ is small the quantiles are recovered (the
dashed lines on the right plot match the theoretical quantiles in plain lines).
It took circa $225$s ($258$ iteration, and $289$ function evaluations) to train
for $\kappa=1\cdot 10^1$, circa $1313$s for $\kappa=1\cdot 10^{-1}$ ($1438$
iterations and $1571$ function evaluations), circa $931$s for $\kappa=1e^{-3}$
($1169$ iterations and $1271$ function evaluations) and $879$s for $\kappa=0$
($1125$ iterations and $1207$ function evaluations). We used a GPU Tensorflow
implementation and run the experiments in float64 on a computer equipped with a
GTX $1070$, and intel i7 $7700$ and $16$Go of DRAM.
\begin{figure}[!htbp]
    \centering
    \resizebox{!}{.5\linewidth}{\includegraphics{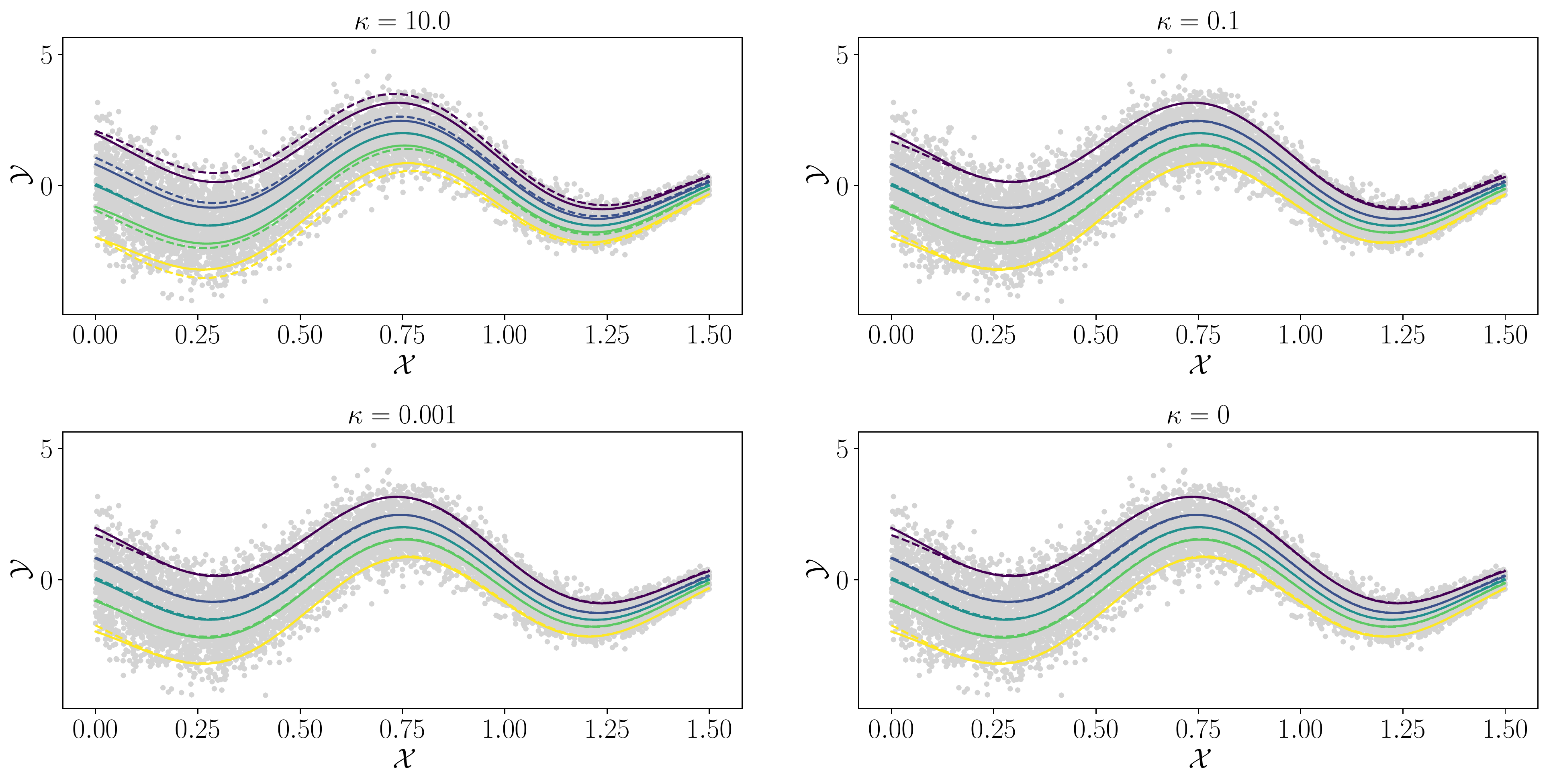}}
    \caption{Impact of the Huber loss smoothing of the pinball loss for
    differents values of $\kappa$. \label{figure:kappa_study}}
\end{figure}
\begin{table}[tb]
    \caption{Examples for objective \eqref{equation:integrated_cost}.
    $\psi_1^\kappa$, $\psi_+^\kappa$: $\kappa$-smoothed absolute value
    and positive part. $h_{x}(\hyperparameter)\defeq h(x)(\theta)$.
    \label{table:integrated_risks}}
    \begin{center}
      \addtolength{\tabcolsep}{-3pt}
      \renewcommand{\arraystretch}{0.8}
        \begin{tiny}
            \begin{sc}
                \resizebox{\textwidth}{!}{%
                \begin{tabular}{lll}
                    \toprule
                    & loss & penalty \\
                    \midrule
                    Quantile  & $\displaystyle\int_{\closedinterval{0}{1}}
                    \abs{\hyperparameter - \indicator{\reals_{-}}(y -
                    h_x(\hyperparameter))}\abs{y - h_x(\hyperparameter)}
                    d\mu(\hyperparameter)$ &
                    $\lambda_{nc}\int_{\closedinterval{0}{1}}
                    \abs{-\frac{dh_x}{d\hyperparameter}(\hyperparameter)}_+
                    d\mu(\hyperparameter) +
                    \frac{\lambda}{2}\norm{h}_{\mcH_K}^2$ \\
                    M-Quantile (smooth)  &
                    $\displaystyle\int_{\closedinterval{0}{1}}
                    \abs{\hyperparameter - \indicator{\reals_{-}}(y -
                    h_x(\hyperparameter))}\psi_1^\kappa\left(y -
                    h_x(\hyperparameter)\right)d\mu(\hyperparameter)$ &
                    $\lambda_{nc}\int_{(0, 1)} \psi_+^\kappa\left(-
                    \frac{dh_x}{d\hyperparameter}(\hyperparameter)\right)
                    d\mu(\hyperparameter) +
                    \frac{\lambda}{2}\norm{h}_{\mcH_K}^2$ \\
                    Expectiles (smooth) &
                    $\displaystyle\int_{\closedinterval{0}{1}}
                    \abs{\hyperparameter - \indicator{\reals_{-}}(y -
                    h_x(\hyperparameter))}\left(y -
                    h_x(\hyperparameter)\right)^2d\mu(\hyperparameter)$ &
                    $\lambda_{nc} \int_{(0, 1)}
                    \abs{-\frac{dh_x}{d\hyperparameter}(\hyperparameter)}_+^2
                    d\mu(\hyperparameter) +
                    \frac{\lambda}{2}\norm{h}_{\mcH_K}^2$ \\
                    Cost-Sensitive & $\displaystyle\int_{
                    \closedinterval{-1}{1}} \abs{\frac{\hyperparameter + 1}{2}
                    - \indicator{\{-1\}}(y)}\abs{1 -
                    yh_{x}(\hyperparameter)}_{+} d\mu(\theta)$ & $
                    \frac{\lambda}{2}\norm{h}_{\mcH_K}^2$ \\
                    Cost-Sensitive (smooth) &
                    $\displaystyle\int_{\closedinterval{-1}{1}}
                    \abs{\frac{\hyperparameter + 1}{2} -
                    \indicator{\{-1\}}(y)}\psi_+^\kappa\left(1 -
                    yh_{x}(\hyperparameter)\right) d\mu(\theta)$ & $
                    \frac{\lambda}{2}\norm{h}_{\mcH_K}^2$ \\
                    Level-Set   &
                    $\displaystyle\int_{\closedinterval{\epsilon}{1}}
                    -t(\hyperparameter) +
                    \frac{1}{\theta}\abs{t(\hyperparameter) -
                    h_x(\hyperparameter)}_+
                    d\mu(\hyperparameter)$ & $
                    \frac{1}{2}\displaystyle\int_{
                    \closedinterval{\epsilon}{1}}
                    \norm{h(\cdot)(\hyperparameter)}_{
                    \mcH_{k_{\inputspace}}}^2 d\mu(\hyperparameter) +
                    \frac{\lambda}{2}\norm{t}_{\mcH_{k_b}}^2$\\ \bottomrule
                \end{tabular}}
            \end{sc}
        \end{tiny}
        \addtolength{\tabcolsep}{3pt}
        \renewcommand{\arraystretch}{0.8}
    \end{center}
\end{table}
\subsection{Experimental protocol for \ac{QR}}  \label{subsection:proto_exp} In
this section, we give additional details regarding the choices being made while
implementing the \ac{ITL} method for $\infty$-\ac{QR}.
\paragraph{\ac{QR} real datasets}
For $\infty$-\ac{QR}, $k_{\inputspace}$, $k_{\hyperparameterspace}$ were
Gaussian kernels. We set a bias term $k_b=k_{\hyperparameterspace}$. The
hyperparameters optimized were $\lambda$, the weight of the ridge penalty,
$\sigma_\inputspace$, the input kernel parameter, and
$\sigma_\hyperparameterspace=\sigma_b$, the output kernel parameter. They were
optimized in the (log)space of $\closedinterval{10^{-6}}{10^{6}}^3$. The
non-crossing constraint $\lambda_{nc}$ was set to $1$. The model was trained on
the continuum $\Theta=\openinterval{0}{1}$ using QMC and Sobol sequences. For
all datasets we draw $m=100$ quantiles form a Sobol sequence \par
For \ac{JQR} we similarly chose two Gaussian kernels. The optimized
hyperparameters were the same as for $\infty$-\ac{QR}.
The quantiles learned were $\theta\in\Set{0.1, 0.3, 0.5, 0.7, 0.9}$.
For the IND-\ac{QR} baseline, we trained independently a non-paramatric
quantile estimator as described in \citet{takeuchi2006nonparametric}. A
Gaussian kernel was used and its bandwidth was optimized in the (log)space of
$\closedinterval{10^{-6}}{10^6}$. No non-crossing was enforced. \par
\subsubsection{Illustration of the \ac{CSC} datasets}
\begin{figure}[!htbp]
\setlength\fboxsep{0pt}\setlength\fboxrule{0.75pt}
\ffigbox[\textwidth]
{
\begin{subfloatrow}[2]
%\fbox{
\ffigbox[.49\textwidth]
  {
    \caption{\textsc{Circle} dataset}
    \label{subfigure:csc_circle}
  }
  {
    \includegraphics[angle=-90,width=\linewidth]{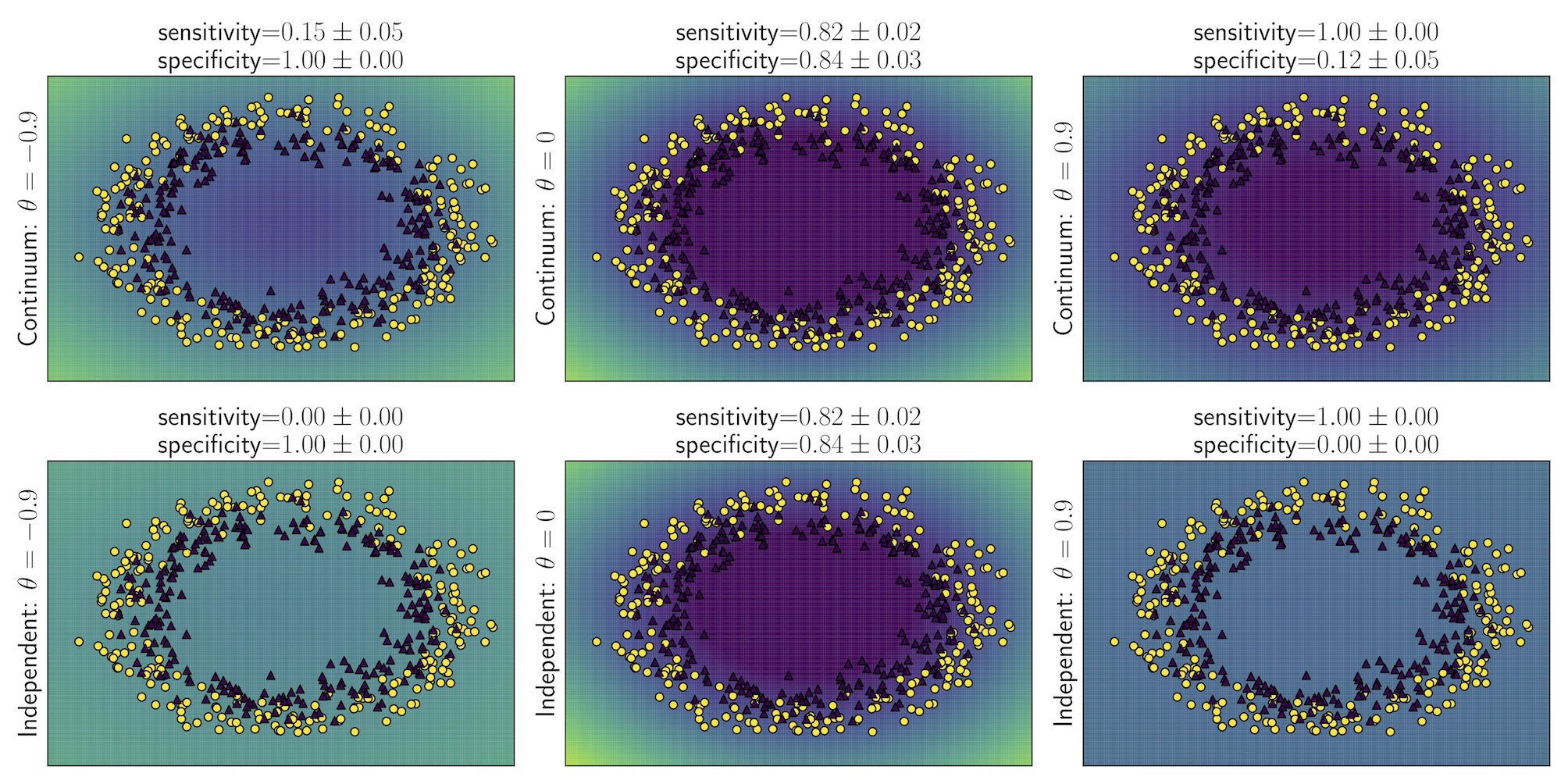}%
  }
%}
%\fbox{
\ffigbox[.4825\textwidth]
  {
    \caption{\textsc{Two-Moons} dataset}
    \label{subfigure:csc_moons}
  }
  {
    \includegraphics[angle=-90,width=\linewidth]{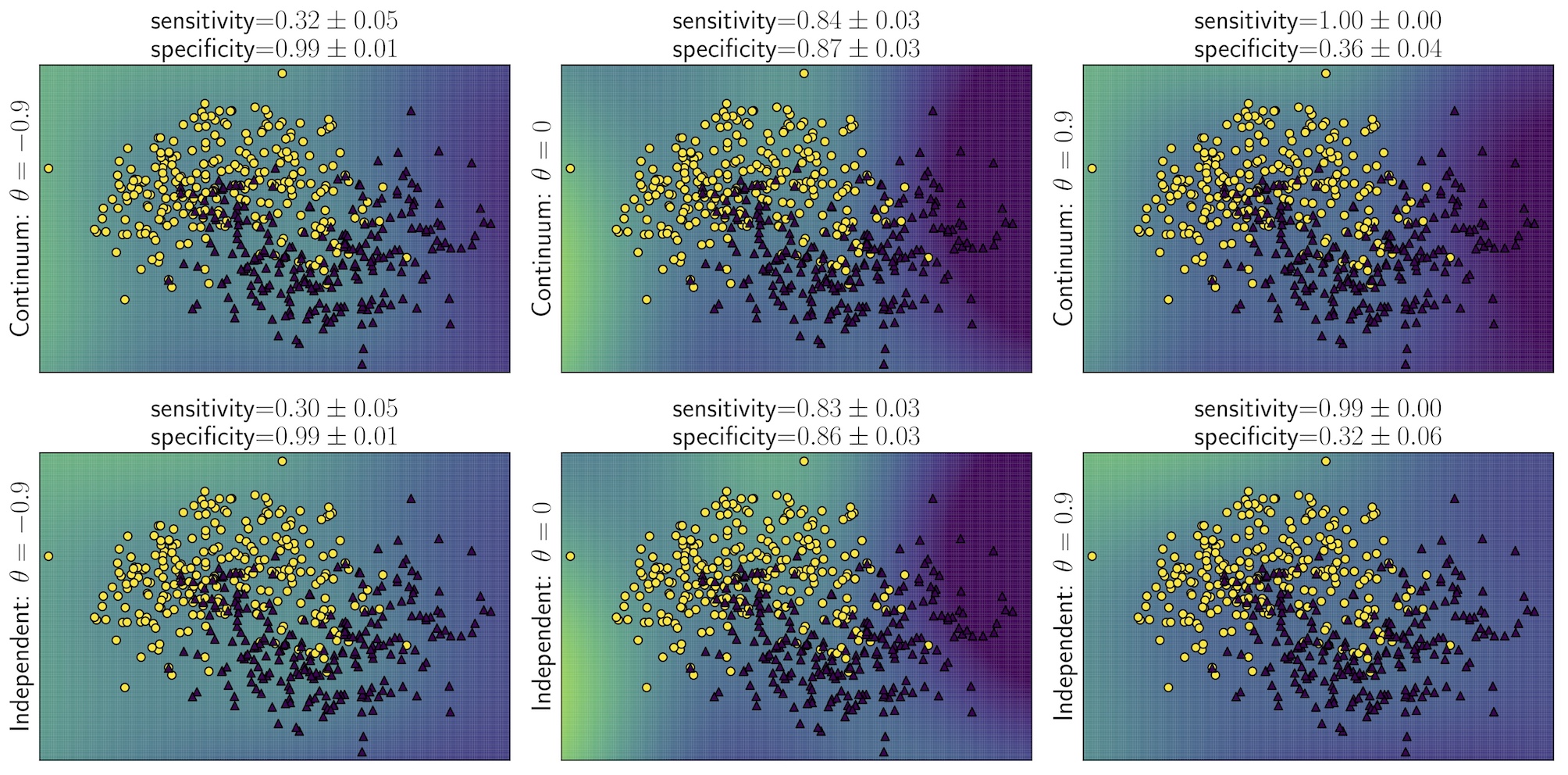}%
  }
%}
\end{subfloatrow}%
}
{
    \caption{Illustration of some datasets used in \cref{table:csc_results}.
    \label{figure:csc_datasets}}
}
\end{figure}%
In this section we illustrate the two classification datasets
\textsc{Two-Moons} and \textsc{Circles} used in \cref{table:csc_results}.
%%%%%%%%%%%%%%%%%%%%%%%%%%%%%%%%%%%%%%%%%%%%%%%%%%%%%%%%%%%%%%%%%%%%%%%%%%%%%%%
\printbibliography[heading=subbibliography]
\end{refsection}

\end{document}